\documentclass{article} %

\usepackage[accepted]{icml2018}

\usepackage{amsmath,amsthm,fancyhdr,lineno}
\usepackage{floatpag}

\usepackage[utf8]{inputenc} 
\usepackage{amsmath, amssymb,bm, cases, mathtools, thmtools}
\usepackage{verbatim}
\usepackage{graphicx}\graphicspath{{figures/}}
\usepackage{multicol}
\usepackage{tabularx}
\usepackage[usenames,dvipsnames]{xcolor}
\usepackage{caption}
\usepackage{mathrsfs} 
\usepackage{algorithm}
\usepackage{algorithmicx}
\usepackage[noend]{algpseudocode}

\usepackage[%
    minnames=1,maxnames=99,maxcitenames=3,
    style=authoryear, %
    doi=true,url=false,
    dashed=false,
    giveninits=true,
    hyperref,natbib,backend=bibtex]{biblatex}
\renewbibmacro{in:}{%
  }
\DeclareNameAlias{sortname}{given-family}
\AtEveryBibitem{\clearfield{issn}}
\AtEveryBibitem{\clearfield{isbn}}
\AtEveryCitekey{\clearfield{issn}}
\AtEveryCitekey{\clearfield{isbn}}
\AtEveryBibitem{\ifentrytype{book}{\clearfield{pages}}{}}
\input{biblatex.cfg}
\bibliography{biblio}

\usepackage[colorlinks,citecolor=black,urlcolor=black,linkcolor=black]{hyperref}
\usepackage{datetime}

\DeclareMathAlphabet\EuRoman{U}{eur}{m}{n}
\SetMathAlphabet\EuRoman{bold}{U}{eur}{b}{n}

\usepackage{inconsolata}

\usepackage[capitalize]{cleveref}

\crefname{lemma}{Lemma}{Lemmas}
\crefname{corollary}{Corollary}{Corollaries}
\crefname{theorem}{Theorem}{Theorems}

\makeatletter
\let\reftagform@=\tagform@
\def\tagform@#1{\maketag@@@{\ignorespaces\textcolor{gray}{(#1)}\unskip\@@italiccorr}}
\renewcommand{\eqref}[1]{\textup{\reftagform@{\ref{#1}}}}
\makeatother

\newcommand{\LATER}[1]{\error}
\newcommand{\fLATER}[1]{\error}
\newcommand{\TBD}[1]{\error}
\newcommand{\fTBD}[1]{}
\newcommand{\PROBLEM}[1]{\error}
\newcommand{\fPROBLEM}[1]{\error}

\declaretheorem[style=plain,numberwithin=section,name=Theorem]{theorem}
\declaretheorem[style=plain,sibling=theorem,name=Lemma]{lemma}

\declaretheorem[style=remark,qed=$\triangleleft$,sibling=theorem,name=Remark]{remark}
\numberwithin{theorem}{section}

\usepackage{mathptmx} 

\def\[#1\]{\begin{align}#1\end{align}}
\def\*[#1\]{\begin{align*}#1\end{align*}}

\def\clap#1{\hbox to 0pt{\hss#1\hss}}

\def\mathclap{\mathpalette\mathclapinternal}

\def\mathclapinternal#1#2{%
\clap{$\mathsurround=0pt#1{#2}$}}

\newcommand{\defas}{\vcentcolon=}  %
\newcommand{\dist}{\ \sim\ }

\newcommand{\Lmax}{L_{\text{max}}}

\newcommand{\Reals}{\mathbb{R}}

\newcommand{\Nats}{\mathbb{N}}

\newcommand{\NNReals}{\Reals_+}

\newcommand{\grad}{\nabla}

\newcommand{\dee}{\mathrm{d}}

\DeclareMathOperator*{\newlim}{\mathrm{lim}\vphantom{\mathrm{infsup}}}

\DeclareMathOperator*{\newmax}{\mathrm{max}\vphantom{\mathrm{infsup}}}
\DeclareMathOperator*{\newinf}{\mathrm{inf}\vphantom{\mathrm{infsup}}}
\DeclareMathOperator*{\newsup}{\mathrm{sup}\vphantom{\mathrm{infsup}}}
\renewcommand{\lim}{\newlim}

\renewcommand{\max}{\newmax}
\renewcommand{\inf}{\newinf}
\renewcommand{\sup}{\newsup}

\newcommand{\ProbMeasures}[1]{\mathcal{M}_1(#1)}

\renewcommand{\Pr}{\mathbb{P}}
\def\EE{\mathbb{E}}

\newcommand{\defn}[1]{\emph{#1}}

\newcommand{\event}[1]{\left \lbrace #1 \right \rbrace}

\newcommand{\KLname}{\mathrm{KL}}
\newcommand{\HH}{\mathcal H}
\newcommand{\KL}[2]{\KLname(#1||#2)}
\newcommand{\KLbin}[2]{\KLname(#1||#2)}

\newcommand{\Normal}{\mathcal N}

\newcommand{\Bernoulli}[1]{\mathcal B_{#1}}

\newcommand{\XX}{\Reals^{\idim}}
\newcommand{\YY}{K}

\newcommand{\loss}{\ell}

\newcommand{\lossbce}{\ell_{\text{BCE}}}

\newcommand{\idim}{k}

\newcommand{\e}{\mathrm{e}}

\newcommand{\randto}{\rightsquigarrow}
\newcommand{\ww}{\mathbf{w}}

\renewcommand{\HH}{\Reals^p}
\newcommand{\RHH}{\ProbMeasures{\HH}}
\newcommand{\Dist}{\mathcal D}
\newcommand{\Alg}{\mathscr A}
\newcommand{\PAlg}{\mathscr P}

\renewcommand{\event}[1]{\bigl ( #1 \bigr )}

\newcommand{\Stest}{S_{\mathrm{tst}}}

\newcommand{\EEE}[1]{\underset{#1}{\EE}}

\newcommand{\set}[2][]{#1\{#2 #1\}}
\newcommand{\brackets}[2][]{#1[#2 #1]}
\newcommand{\parens}[2][]{#1(#2 #1)}

 \renewcommand{\event}[2][]{#1(#2 #1)}

\renewcommand{\defas}{\overset{\text{\smash{\tiny{def}}}}{=}}

\newcommand\optparen[1]{\ifthenelse{\equal{#1}{}}{}{(#1)}}
\newcommand{\RiskChar}{R}
\newcommand{\Risk}[2]{\RiskChar_{#1}\optparen{#2}}
\newcommand{\EmpRisk}[2]{\hat \RiskChar_{#1}\optparen{#2}}

\newcommand{\SurEmpRisk}[2]{\tilde \RiskChar_{#1}\optparen{#2}}

\newcommand{\ESGLD}{Entropy-SGLD}
\newcommand{\GibbsP}{local entropy}

\newcommand{\binloss}{\raisebox{1pt}{\scalebox{0.75}{\textrm{\normalfont\scriptsize \,0--1\!}}}}
\newcommand{\Err}[2]{\RiskChar^{\binloss}_{#1}\optparen{#2}}
\newcommand{\ErrRamp}[2]{\RiskChar^{r}_{#1}\optparen{#2}}
\newcommand{\EmpErr}[2]{\RiskChar^{\binloss\!}_{#1}\optparen{#2}}
\newcommand{\wdp}{\ww^{\ast}}

\newcommand{\Pw}[1]{P_{#1}}
\newcommand{\Qw}[1]{Q_{#1}^{\,\smash{S}}}%

\usepackage{tikz}
\usetikzlibrary{decorations.pathreplacing} 

\usepackage{enumitem}
\usepackage[T1]{fontenc}    %
\usepackage[utf8]{inputenc} %
\usepackage{url}            %
\usepackage{booktabs}       %
\usepackage{amsfonts}       %

\usepackage{xifthen}

\icmltitlerunning{Entropy-SGD optimizes the prior of a PAC-Bayes bound}

\begin{document}

\twocolumn[
\icmltitle{
Entropy-SGD optimizes the prior of a PAC-Bayes bound:\\
Generalization properties of Entropy-SGD and data-dependent priors
}

\icmlsetsymbol{equal}{*}

\renewcommand{\thefootnote}{\fnsymbol{footnote}}

\begin{icmlauthorlist}
\icmlauthor{Gintare Karolina Dziugaite}{cam,vec}
\icmlauthor{Daniel M.~Roy}{to,vec}
\end{icmlauthorlist}

\icmlaffiliation{cam}{Dept.\ of Engineering, Univ.\ of Cambridge, Cambridge, UK}
\icmlaffiliation{to}{Dept.\ of Statistical Sciences, Univ.\ of Toronto, Toronto, Canada}
\icmlaffiliation{vec}{Vector Institute, Toronto, Canada}

\icmlcorrespondingauthor{Gintare Karolina Dziugaite}{gkd22@cam.ac.uk}
\icmlcorrespondingauthor{Daniel M.~Roy}{droy@utstat.toronto.edu}

\icmlkeywords{}

\vskip 0.3in
]

\printAffiliationsAndNotice{}  %
\begin{abstract}
We show that Entropy-SGD \citep{CCSL16},
when viewed as a learning algorithm, 
optimizes a PAC-Bayes bound on the risk of a Gibbs (posterior) classifier, i.e.,
a randomized classifier obtained by a risk-sensitive perturbation of the weights of a learned classifier.
Entropy-SGD works by optimizing the bound's prior,
violating the hypothesis of the PAC-Bayes theorem that the prior is chosen independently of the data. 
Indeed, available implementations of Entropy-SGD 
rapidly obtain zero training error on random labels and the same holds of the Gibbs posterior.
In order to obtain a valid generalization bound, 
we rely on a result showing that 
data-dependent priors obtained by 
stochastic gradient Langevin dynamics (SGLD)
yield valid PAC-Bayes bounds provided the target distribution of SGLD is $\epsilon$-differentially private.
We observe that test error on MNIST and CIFAR10 %
falls within the (empirically nonvacuous) risk bounds 
computed under the assumption that SGLD reaches stationarity.
In particular, 
Entropy-SGLD can be configured to yield relatively tight generalization bounds and still fit real labels,
although these same settings do not obtain state-of-the-art performance.
\end{abstract}

\newcommand{\Epochs}{Epochs $\div$ L}

\newcommand{\LGibbs}{G_{\gamma,\tau}^{\ww,\smash{S}}}
\newcommand{\LGibbsDen}{g^{\ww,\smash{S}}_{\gamma,\tau}}

\newcommand{\LEDist}[1]{P_{\exp \parens {#1 \, F_{\gamma,\tau}(\cdot; S) }}}

\section{Introduction}

Optimization is central to much of machine learning, but generalization is the ultimate goal.
Despite this, the generalization properties of many optimization-based learning algorithms 
are poorly understood.
The standard example is stochastic gradient descent (SGD), one of the workhorses of deep learning, 
which has good generalization performance
in many settings, even under overparametrization \citep{Ney1412},
but rapidly overfits in others \citep{Rethinking17}.
Can we develop high performance learning algorithms with provably strong generalization guarantees? 
Or is there a limit?

In this work, we study an optimization algorithm called Entropy-SGD \citep{CCSL16},
which was designed to outperform SGD in terms of generalization error when optimizing an empirical risk. 
Entropy-SGD minimizes an objective $f : \HH \to \Reals$
indirectly by approximating stochastic gradient ascent on 
the so-called local entropy 
\begin{equation*}
F(\ww) 
\defas C(\tau) + \log \,\underbrace{\EE_{\xi \sim \Normal_\tau} \brackets{e^{-\tau\, f(\ww+\xi)}}}_{\mathclap{\smash{\int \exp(-f(\ww+x)) \Normal_\tau\!(\dee x)}}}
\end{equation*}
where $\tau > 0$ is an inverse temperature, $C(\tau)$ is an additive constant, and $\Normal_\tau$ denotes a zero-mean isotropic multivariate normal distribution on $\HH$ whose scale depends on $\tau$.

Our first contribution is connecting Entropy-SGD to results in statistical learning theory,
showing that maximizing the local entropy corresponds to minimizing a PAC-Bayes bound \citep{PACBayes} 
on the risk of the so-called Gibbs posterior.
The distribution of $\ww + \xi$ is the PAC-Bayesian ``prior'', 
and so optimizing the local entropy optimizes the bound's prior.
This connection between local entropy and PAC-Bayes 
follows from a result due to \citet[][Lem.~1.1.3]{Catoni} in the case of bounded risk. (See \cref{optprior}.)
In the special case where $\tau f$ is the empirical cross entropy, 
the local entropy is literally a Bayesian log marginal density. 
The connection between minimizing PAC-Bayes bounds under log loss and maximizing log marginal densities 
is the subject of recent work by \citet{germain2016pac}. Similar connections have been made by \citet{zhang2006,zhang2006information,grunwald2012safe,GrunwaldM16}.

Despite the connection to PAC-Bayes,
as well as theoretical results by \citeauthor{CCSL16} suggesting that Entropy-SGD may be more stable than SGD,
we demonstrate that Entropy-SGD (and its corresponding Gibbs posterior) can rapidly overfit, just like SGD. 
We identify two changes, motivated by theoretical analysis, that %
prevent overfitting.
The first change relates to the stability 
of the optimized prior mean, with respect to changes to the data.
The PAC-Bayes theorem requires that the prior be independent of the data,
and so by optimizing the prior mean, Entropy-SGD invalidates the bound.
Indeed, the bound does not hold empirically. 
While a PAC-Bayes prior may not be chosen based on the data, %
it can depend on the data distribution.  
This suggests that if the prior depends only weakly on the data, 
it may be possible to derive a valid bound and control overfitting.

Indeed, \citet{DR18private} recently formalized this idea using 
differential privacy \citep{Dwork2006,dwork2015preserving}
under the assumption of bounded risk.
Using existing results connecting statistical validity and differential privacy \citep[Thm.~11]{dwork2015preserving},
they show that an $\epsilon$-differentially private prior yields a valid, 
though looser, PAC-Bayes bound. 

Achieving strong differential privacy can be computationally intractable.
Motivated by this obstruction, \citeauthor{DR18private} relax the privacy requirement in the case of Gaussian PAC-Bayes priors parameterized by their mean vector.
They show that convergence in distribution to a differentially private mechanism suffices for generalization.
This allows one to use  
stochastic gradient Langevin dynamics \citep[SGLD;][]{welling2011bayesian},
which is known to converge weakly to its target distribution, under regularity conditions.
We will refer to the Entropy-SGD algorithm as \ESGLD\ when the SGD step on local entropy is replaced by SGLD.

The one hurdle to using data-dependent priors learned by SGLD is that we cannot easily measure how close
we are to converging. Rather than abandoning this approach, we take two steps: 
First, we run SGLD far beyond the point where it appears to have converged.
Second, we assume convergence, but then view/interpret the bounds as being optimistic. 
In effect, these two steps allow us to see the potential and limitations of using private data-dependent priors to study \ESGLD.

Empirically, we find that the resulting PAC-Bayes bounds are quite tight but still conservative.
On MNIST, when the limiting privacy of \ESGLD{} is tuned to contribute no more than $ 2 \epsilon^2 \times 100 \approx 0.2\%$ to the generalization error,
the test-set error of the learned network is 3--8\%, which is roughly 5--10 times higher than state-of-the-art test-set error, which for MNIST is between 0.2-1\%.\footnote{%
These numbers must be interpreted carefully---the simple fact that the deep-learning tool chain was developed using MNIST likely implies that generalization and test set bounds are biased.} 

The second change pertains to the stability of the stochastic gradient estimate made on each iteration of Entropy-SGD. 
This estimate is made using SGLD.  
(Hence Entropy-SGD is SGLD within SGD.)
\citeauthor{CCSL16} make a subtle but critical modification to the noise term in the SGLD update: 
the noise is divided by a factor that ranges from $10^3$ to $10^4$.
(This factor was ostensibly tuned to produce good empirical results.)
Our analysis shows that, as a result of this modification, 
the Lipschitz constant of the objective function is approximately $10^6$--$10^8$ times larger, 
and the conclusion that the Entropy-SGD objective is smoother than the original risk surface no longer stands.
This change to the noise also negatively impacts the differential privacy of the prior mean. 
Working backwards from the desire to obtain tight generalization bounds, 
we are led to divide the SGLD noise by a factor of only $\sqrt[4]{m}$, where $m$ is the number of data points. (For MNIST, $\sqrt[4]{m} \approx 16$.) 
The resulting bounds are nonvacuous and tighter than those recently published by \citet{DR17}, 
although it must be emphasized that 
the bounds are optimistic because we assume SGLD has converged. The extent to which it has not converged may
inflate the bound.

We begin by introducing sufficient background so that we can make a formal connection  
between local entropy and PAC-Bayes bounds. We discuss additional related work in \cref{sec:rel}.
We then introduce several existing learning bounds that use differential privacy,
including the PAC-Bayes bounds outlined above that use data-dependent priors.
In \cref{sec:eval}, we present experiments on MNIST and CIFAR10, which provide evidence for our theoretical analysis. 
We close with a short discussion.

\section{Preliminaries: Supervised learning, Entropy-SGD, and PAC-Bayes}

We consider the batch supervised learning setting, where we are given a sample $z_1,\dots,z_m$ 
drawn i.i.d.\ from an unknown probability distribution $\Dist$ on a space $Z = X \times Y$ of labeled examples.
Given a family of classifiers, indexed by weight vectors $\ww \in \HH$, and 
a bounded 
loss function $\loss: \HH \times Z \to \Reals$.
the \defn{risk} and \defn{empirical risk} are 
\begin{equation*}
\Risk{\Dist}{\ww} \defas \smash{\EEE{z \sim \Dist} \event {\loss(\ww,z) }} 
;\quad
\textstyle \EmpRisk{S}{Q} 
\defas 
\smash {\frac 1 m \sum_{i=1}^m \EEE{\ww \sim Q} \event {\loss(\ww,z_i)}}.
\end{equation*}
Our goal is to learn a classifier with small risk, taking advantage of the fact that
$\Risk{\Dist}{\ww} = \EE_{S \sim \Dist^m} \brackets{\EmpRisk{S}{\ww}}$.
We consider randomized (Gibbs) classifiers,
formalized as probability distributions $Q \in \RHH$ on the space of weight vectors.
The (expected) risk of a randomized classifier is %
\[
\Risk{\Dist}{Q} 
\defas 
\smash
{\EEE{\ww \sim Q} \event {\Risk{\Dist}{\ww} }
= \EEE{z \sim \Dist} \event[\big]{ \EEE{\ww \sim Q} \event {\loss(\ww,z)} } }.
\]
We will sometimes refer to elements of $\HH$ and $\RHH$ as classifiers and randomized classifiers, respectively.

Our focus is the case of neural network that output probability vectors $p(\ww,x)=(p(\ww,x)_1,\dots,p(\ww,x)_K)$ over 
$K$ classes on input $x$ when the weights are $\ww$.
Zero--one (0--1) loss is $\loss(\ww,(x,y)) = 1$ if  $y = \arg\max_{k} p(\ww,x)_{k}$ and $0$ otherwise.
We also use cross entropy loss as a differentiable surrogate. 
Cross entropy loss is $\loss(\ww,(x,y)) = -\log p(\ww,x)_{y}$.
Note that cross entropy loss is merely bounded below.
We use a bounded modification (\cref{apphyper}).
We will often refer to the (empirical) 0--1 risk as the (empirical) error.

\subsection{Entropy-SGD}
\label{sec:ent}

Entropy-SGD is a gradient-based learning algorithm 
proposed by \citet{CCSL16} as an alternative to stochastic gradient descent on the empirical risk surface $\EmpRisk{S}{}$.
The authors argue that Entropy-SGD has better generalization performance.
Part of that argument is a theoretical analysis of the smoothness of the local entropy surface that Entropy-SGD optimizes in place of the empirical risk surface, as well as a uniform stability argument that they admit rests on assumptions that are violated, but to a small degree empirically.
As we have mentioned in the introduction, Entropy-SGD's modifications to the noise term in SGLD result in much worse smoothness. 
We will modify Entropy-SGD in order to stabilize its learning and control overfitting.

Entropy-SGD is stochastic gradient ascent 
applied to the optimization problem
$\arg \max_{\mathstrut\ww  \in \HH}  F_{\gamma,\tau}(\ww; S)$, 
where
\begin{equation}\label{localentdef}
\vphantom{\underbrace{w}_{y}} F_{\gamma,\tau}(\ww; S) 
=  \log \int_{\Reals^p} 
                 \smash {\underbrace {\!\! \exp \event { - \tau \EmpRisk{S}{\ww'} - \tau \frac {\gamma}{2} \| \ww' - \ww \|_2^2 }}_{\LGibbsDen(\ww')}}
                 \,\dee \ww'.
\end{equation}
The objective $F_{\gamma,\tau}(\cdot;S)$ is known as the \emph{local entropy}, 
and can be viewed as the log partition function of 
the unnormalized probability density function $\LGibbsDen$.
(We will denote the corresponding distribution by $\LGibbs$.)
Assuming that one can exchange differentiation and integration, 
it is straightforward to verify that
\begin{equation}\label{derivid}
\grad_{\ww} F_{\gamma,\tau}(\ww; S)
= \smash{ \EEE{\ww' \sim \LGibbs} \event{ \tau\gamma(\ww-\ww')} },
\end{equation}
and then the local entropy $F_{\gamma,\tau}(\cdot;S)$ is differentiable, even if the empirical risk $\EmpRisk{S}{}$ is not.
Indeed, \citeauthor{CCSL16} show that the local entropy and its derivative are Lipschitz.
\citeauthor{CCSL16} argue informally that maximizing the local entropy leads to ``flat minima'' in the empirical risk surface, 
which several authors \citep{Hinton93,Hochreiter97,PhysRevLett.115.128101,BBCetal16} have argued is tied to good generalization performance (though none of these papers gives generalization bounds, vacuous or otherwise). 
\citeauthor{CCSL16} propose \emph{approximate} SGD on the local entropy,
replacing the gradient $\grad_{\ww} F_{\gamma,\tau}(\ww; S)$ with 
a Monte Carlo estimate
$\tau \gamma (\ww - \mu_L)$,
with $\mu_1 = \ww_1$ and $\mu_{j+1} = \alpha \ww'_j + (1-\alpha) \mu_j$,
where $\ww'_1,\ww'_2,\dots$ are (approximately) i.i.d.\ samples from $\LGibbs$ 
and $\alpha \in (0,1)$ defines a weighted average.
Obtaining samples from $\LGibbs$ is likely intractable when the dimensionality of the weight vector is large.
The authors assume the empirical risk is differentiable and use
Stochastic Gradient Langevin Dynamics \citep[SGLD;][]{welling2011bayesian},
which  simulates a Markov chain whose long-run distribution converges to $\LGibbs$.\footnote{
\citeauthor{CCSL16} take $L=20$ steps of SLGD, 
using a constant step size $\eta'_j = 0.2$ %
and weighting $\alpha=0.75$.
}
The final output of Entropy-SGD is the deterministic predictor corresponding to the final weights $\ww^*$ achieved by several epochs of optimization.

\begin{algorithm}[t]
    \caption{One outerloop step of Entropy-SG(L)D}%
    \label{entropysgldalg}
    \begin{algorithmic}[1] %
     \Require
     \Statex $\ww \in \HH$ \Comment{Current weights}
     \Statex $S \in Z^m$ \Comment{Data}    
     \Statex $\loss : \HH \times Z \to \Reals$ \Comment{Loss}   
     \Statex $\tau,\beta,\gamma,\eta,\eta',L,K$ \Comment{Parameters}
     \Ensure Weights $\ww$ moved along stochastic gradient
        \Procedure{Entropy-SG(L)D-step}{} 
     \State $\ww', \mu \gets \ww$
            \For{$i \in \{1,...,L \}$ } \Comment{Run SGLD for L iterations.}
            	\State $\eta'_i \gets \eta' / i $
                 \State $(z_{j_1},\dots,z_{j_K})  \gets$ sample minibatch of size $K$%
                 \State $\dee\ww' \gets  
                                - \frac \tau K \sum_{i=1}^K \grad_{\ww'} \loss(\ww', z_{j_i}) - \gamma \tau (\ww' -\ww)$
                 \State $\ww' \gets \ww' + \frac{1}{2} \eta'_i \dee\ww'  + \sqrt{\eta'_i}  N(0,I_{p})$
                 \State $\mu \gets (1-\alpha) \mu + \alpha \ww'$
            \EndFor
            \Comment{C.f.~\cref{derivid}.}
            \State $\ww \gets \ww + \frac 1 2 \eta \tau \gamma (\ww -  \mu) + \smash{\underbrace{\sqrt{\eta/\beta} \, N(0,I_{p})}_{\text{Entropy-SGLD only}}}$ 
            \State \textbf{return} $\ww \vphantom{\underbrace{W}}$
        \EndProcedure
    \end{algorithmic}
\end{algorithm}

\cref{entropysgldalg} 
gives a complete description of the stochastic gradient step performed by Entropy-SGD. %
If we rescale the learning rate, $\eta' \gets \frac 1 2 \eta' \tau $, lines 6 and 7 are equivalent to 
\begin{align*}
\scalebox{.9}{\textcolor{black}{$6$:}}&\qquad\textstyle \dee\ww' \gets  
	- \frac 1 K \sum_{i=1}^K \grad_{\ww'} \loss(\ww', z_{j_i}) - \gamma (\ww' -\ww) \\
\scalebox{.9}{\textcolor{black}{$7$:}}&\qquad\textstyle \ww' \gets \ww' + \eta'_i \dee\ww'  + \sqrt{\eta'_i}\sqrt{2/\tau}\, N(0,I_{p})
\end{align*}
Notice that the noise term is multiplied by a factor of $\sqrt{2/\tau}$. 
A multiplicative factor $\epsilon$---called the ``thermal noise'', but playing exactly the same role as $\sqrt{2/\tau}$ here---%
appears in the original description of the Entropy-SGD algorithm given by \citeauthor{CCSL16}.
However, $\epsilon$ does not appear in the definition of local entropy used in their stability analysis.  
Our derivations highlight that scaling the noise term in SGLD update has a profound effect: 
the thermal noise exponentiates the density that defines the local entropy. 
The smoothness analysis of Entropy-SGD does not take into consideration the role of $\epsilon$, 
which is critical because \citeauthor{CCSL16} take $\epsilon$  to be as small as $10^{-3}$ and $10^{-4}$.
Indeed, the conclusion that the local entropy surface is smoother no longer holds.
We will see that $\tau$ controls the stability (and then the generalization error) of our variant of Entropy-SGD. 

\subsection{KL divergence and the PAC-Bayes theorem}

Let $Q,P \in \ProbMeasures{\HH}$,
assume $Q$ is absolutely continuous with respect to $P$, 
and write $\frac{\dee Q}{\dee P} : \HH \to \NNReals \cup \{\infty\}$ for some Radon--Nikodym derivative of $Q$ with respect to $P$.
Then the Kullback--Liebler divergence 
from $Q$ to $P$ is 
\begin{equation*}
\KL{Q}{P} 
\defas 
\smash{ \int \log \frac{\dee Q}{\dee P} \,\dee Q} .
\end{equation*}
Let $\Bernoulli{p}$ denote the Bernoulli distribution on $\{0,1\}$ with mean $p$.
For $p,q \in [0,1]$, we abuse notation and define
\begin{equation*}
\KLbin{q}{p} \defas 
\smash{ \KL{\Bernoulli{q}}{\Bernoulli{p}} = q \log \frac q p + (1-q) \log \frac {1-q}{1-p}. }
\end{equation*}

We now present a PAC-Bayes theorem. The first such result was established by \citet{PACBayes}.
We focus on the setting of bounding the generalization error of a (randomized) classifier on a finite discrete set of labels $K$.
We will use the following variation of a PAC-Bayes bound, where we consider bounded loss functions.
\begin{theorem}[Linear PAC-Bayes Bound; {\citealt{McA13, Catoni}}]
\label{pacbayeslinear}
Fix $\lambda > 1/2$ %
and assume the loss takes values in an interval of length $\Lmax$.
For every $\delta > 0$, $m \in \Nats$, 
$\Dist \in \ProbMeasures{\XX \times \YY}$, 
and 
$P \in \ProbMeasures{\HH}$,
with probability at least $1-\delta$, 
for every $Q \in \ProbMeasures{\HH}$,
\begin{equation}\label{linearpbound}
       \Risk{\Dist}{Q}
      			\le \smash {\frac{1}{1-\frac{1}{2 \lambda}}  \parens[\big]{ \EmpRisk{S}{Q}  + 
					\frac{\lambda \Lmax }{m}   \parens[\big]{ \KL{Q}{P} + \log \frac {1}{\delta} } }. }
\end{equation}
\end{theorem}

The (PAC-Bayes) prior $P$ in the bound is data independent.
Later, we introduce bounds for data-dependent priors.

\section{Maximizing local entropy minimizes a PAC-Bayes bound}
\label{localentasPACBayes}
\newcommand{\rr}{r}

We now present our first contribution, a connection between the local entropy and PAC-Bayes bounds.
We begin with some notation for Gibbs distributions.
For a measure $P$ on $\HH$ and function $g : \HH \to \Reals$, let $P[g]$ denote the expectation $\int g(h) P(\dee h)$ and, provided $P[g] < \infty$, let $P_{g}$ denote the probability measure on $\HH$, absolutely continuous with respect to $P$, with Radon--Nikodym derivative 
$
\frac{\dee P_{g}}{\dee P}(h) = 
\frac{g(h)}{P[g]}.
$
A distribution of the form $P_{\exp \parens { - \tau g }}$  is generally referred to as a Gibbs distribution. 
In the special case where $P$ is a probability measure, 
we call $P_{\exp \parens{ -\tau \EmpRisk{S}{} }}$ a ``Gibbs posterior''.

\begin{theorem} [Maximizing local entropy optimizes a PAC-Bayes bound's prior] 
\label{optprior}
Assume the loss takes values in an interval of length $\Lmax$, 
let $\tau = \frac{m}{\lambda \Lmax}$ for some $\lambda>1/2$, 
Then the set of weight $\ww$ maximizing the local entropy $F_{\gamma,\tau}(\ww;S)$ 
equals the set of weights $\ww$ minimizing the right hand side of \cref{linearpbound}
for $Q = \LGibbs = P_{\exp \parens{- \tau \EmpRisk{S}{}}}$
and $P$ a multivariate normal distribution with mean $\ww$ and covariance matrix $(\tau\gamma)^{-1} I_{p}$.
\end{theorem}
See \cref{localpacbayes} for the proof.
The theorem requires the loss function to be bounded, because the PAC-Bayes bound we have used applies only to bounded loss functions. 
\citet{germain2016pac} described PAC-Bayes generalization bounds for unbounded loss functions, though it requires that one make additional assumptions about the distribution of the empirical risk, which we would prefer not to make.
(See \citet{GrunwaldM16} for related work on excess risk bounds and further references).

\section{Data-dependent PAC-Bayes priors}

\cref{optprior} reveals that Entropy-SGD is optimizing a PAC-Bayes bound with respect to the prior. 
As a result, the prior $P$ depends on the sample $S$, 
and the hypotheses of the PAC-Bayes theorem (\cref{pacbayeslinear}) are not met.  
Naively, it would seem that this interpretation of Entropy-SGD cannot explain its ability to generalize. 
Using tools from differential privacy,  %
\citet{DR18private} show that if the prior term is optimized in a differentially private way,  
then a PAC-Bayes theorem still holds, at the cost of a slightly looser bound.
We will assume basic familiarity with differential privacy. (See \citet{DR18private} for a basic summary.) 
We borrow the notation $\Alg\colon Z \randto T$ for a (randomized) algorithm with an input in $Z$ and output in $T$.

The key result is due to \citet[Thm.~11]{dwork2015preserving}.

\begin{theorem}
\label{dpthm}
Let $m \in \Nats$,
let $\Alg\colon Z^m \randto T$,
let $\Dist$ be a distribution over $Z$,
let $\beta \in (0,1)$,
and, for each $t \in T$, fix a set $v(t) \subseteq Z^m$
such that
$
\Pr_{S \sim \Dist^m} \event{ S \in v(t) } \le \beta.
$
If $\Alg$ is $\epsilon$-differentially private for 
$
\epsilon \le \sqrt { {\ln (1/\beta)}/{(2 m)}},
$
then
$
\Pr_{S \sim \Dist^m} %
   \event{  S \in v(\Alg(S)) } \le 3 \sqrt {\beta} .
$
\end{theorem}

Using \cref{dpthm}, one can compute tail bounds on the generalization error of fixed classifiers, and then,
provided that a classifier is learned from data in a differentially private way, 
the tail bound holds on the classifier, with less confidence. 
The following two tail bounds are examples of this idea 
due to \citet[Lem.~2 and Lem.~3]{Oneto2017}.  
\begin{theorem}
\label{onetobounds}
Let $m \in \Nats$, 
let $\Alg\colon Z^m \randto \HH$ be $\epsilon$-differentially private,
and let $\delta > 0$.
Then $| \Risk{\Dist}{\Alg(S)}- \EmpRisk{S}{\Alg(S)} | <
\bar\epsilon + m^{-\frac 1 2}$ 
with probability at least $1-\delta$ over $S \sim \Dist^m$,
where $\bar\epsilon = \max \set{\epsilon, \sqrt{\frac 1 m \log \frac 3 \delta}}$.
The same holds for the upper bound
  $\surd\parens{6 \EmpRisk{S}{\Alg(S)} }(\bar\epsilon + m^{-\frac 1 2 }) + 6(\bar\epsilon^2+m^{-1})$.
\end{theorem}

\subsection{An $\epsilon$-differentially private PAC-Bayes bound}
\label{sec:exp}

The PAC-Bayes theorem allows one to choose the prior based on the data-generating distribution $\Dist$, but 
not on the data $S \sim \Dist^m$.
Using differential privacy, one can consider a data-dependent prior $\PAlg(S)$.
\begin{theorem}[{\citealt{DR18private}}]
\label{DPpacbayes}
Under 0--1 loss, 
for every $\delta > 0$, $m \in \Nats$, 
$\Dist \in \ProbMeasures{\XX \times \YY}$, 
and $\epsilon$-differentially private data-dependent prior $\PAlg \colon Z^m \randto \ProbMeasures{\HH}$,
with probability at least $1-\delta$ over $S \sim \Dist^m$,
for every $Q \in \ProbMeasures{\HH}$,
\begin{align}
\begin{split}\label{DPpacbound}
      &
      \KLbin{\EmpRisk{S}{Q}}{\Risk{\Dist}{Q}} 
      \\&\quad \le \frac {\KL{Q}{\PAlg(S)} + \ln 2\sqrt{m} + 
              2 \max \{
                 \ln \frac {3}{\delta}, \ 
                 m \epsilon^2 
                       \} 
           }{m} 
           .
\end{split}           
\end{align}
\end{theorem}
\vspace*{-1em}
Note that the bound holds for any posterior $Q$, 
including one obtained by optimizing a \emph{different} PAC-Bayes bound. 
Inverting $\KLbin{\EmpRisk{S}{Q}}{\Risk{\Dist}{Q}}$ allows one to obtain a two-sided confidence interval for $\Risk{\Dist}{Q}$.
Note that, in realistic scenarios, $\delta$ is large enough relative to $\epsilon$ that
an $\epsilon$-differentially private prior $\PAlg(S)$ contributes $2\epsilon^2$ to the generalization error.
Therefore, $\epsilon$ must be much less than one to not contribute a nontrivial amount to the generalization error.
As discussed by \citeauthor{DR18private}, one can match the $m^{-1}$ rate by which the KL term decays choosing $\epsilon \in O(m^{-1/2})$.  
Our empirical studies use this rate.

\subsection{Differentially private data-dependent priors}
\label{gibbssampling}

We have already explained that the weights learned by Entropy-SGD can be viewed as 
the mean of a data-dependent prior $\PAlg(S)$.
By \cref{DPpacbayes} and the fact that post-processing does not decrease privacy, 
it would suffice to establish that the mean is $\epsilon$-differentially private
in order to obtain a risk bound on the corresponding Gibbs posterior classifier.

The standard (if idealized) approach for optimizing a data-dependent objective in a private way 
is to use the exponential mechanism \citep{ExpRelease07}. 
In the context of maximizing the local entropy, the exponential mechanism corresponds to sampling exactly from the ``local entropy (Gibbs)
distribution'' 
$\LEDist{\beta}$, %
where $\beta > 0$ and $P$ is some measure on $\HH$. (It is natural to take $P$ to be Lebesgue measure, or a multivariate normal distribution, which would correspond to L2 regularization of the local entropy.)
The following result establishes the privacy of a sample from the local entropy distribution:

\begin{theorem}\label{dpGibbspSample}
Let $\gamma,\tau > 0$, 
and assume the range of the loss is contained in an interval of length $\Lmax$.
One sample from the \GibbsP\ distribution $\LEDist{\beta}$ is $\frac{2 \beta \Lmax \tau}{m}$-differentially private.
\end{theorem}

See \cref{localentdist} for proof.
Sampling from exponential mechanisms exactly is generally intractable.
We therefore rely on the following result due to \citet{DR18private},
which allows us to use SGLD to produce an approximate sample and obtain the same bound up to a term that depends on the degree of convergence. 
Let $\Err{\Dist}{\cdot}$ denote risk with respect to 0--1 loss
and let $\Risk{\Dist}{\cdot}$ denote risk with respect to the bounded version of cross-entropy described by \citet{DR18private}. 
 (For completeness, the bounded version is defined in \cref{binobjective}.)

\begin{theorem}[SGLD PAC-Bayes Bound]\label{maindppacthm}
Let $\tau > 0$ and $\Sigma \in \Reals^{p \times p}_{\succeq 0}$.
For $\ww \in \HH$ and $S \in Z^m$,
let $\Pw{\ww} = \Normal(\ww,\Sigma)$, 
$\Qw{\ww} = (\Pw{\ww})_{\smash{\exp(-\tau \EmpRisk{S}{})}\vphantom{\ww}}$,
and assume $\EmpRisk{S}{\cdot}$ is bounded.
Then, 
for every $\epsilon' > 0$ and $\delta,\delta' \in (0,1)$,
with probability at least $1-\delta-\delta'$
over $S \dist \Dist^{m}$ and
a sequence $\ww_1,\ww_2,\dots$ (such as produced by SGLD) 
converging in distribution (conditionally on $S$) to an $\epsilon$-differentially private vector $\wdp(S)$,
there exists $N \in \Nats$, such that,
for all $n > N$,
\begin{align*}
\begin{split}%
      &
      \KL{\EmpErr{S}{\Qw{\ww_n}}}{\Err{\Dist}{\Qw{\ww_n}}} 
      \\&\quad \le \frac {\KL{\Qw{\ww_n}}{\Pw{\ww_n}} + \ln 2\sqrt{m} + 
              2 \max \{
                 \ln \frac {3}{\delta}, \ 
                 m \epsilon^2
                       \} 
           }{m} 
           + \epsilon'.
\end{split}           
\end{align*}
\end{theorem}
\begin{remark}
\citet{RRT17} give conditions that suffice to imply that SGLD converges in distribution. 
Note that the number of required iterations $N$ of SGLD  may depend on the sample $S$, $\epsilon$, and $\delta$.
See \citep{DR18private} for details and improved bounds.
\end{remark}

In summary, 
we optimize the local entropy $F_{\gamma,\tau}(\cdot; S)$ using SGLD,
repeatedly performing the  update
\begin{equation*}
\textstyle \ww \gets \ww + \frac 1 2 \eta \hat{g}(\ww) + \sqrt{\eta/\beta} \,N(0,I_{p}),
\end{equation*}
where at each round $\hat{g}(\ww)$ is an estimate of the gradient 
$\grad_{\ww} F_{\gamma,\tau}(\ww; S)$. (Recall the identity \cref{derivid}.)
As in Entropy-SGD, we construct biased gradient estimates via an inner loop of SGLD. 
Ignoring error from these biased gradients,
we obtain a data-dependent prior that yields a valid PAC-Bayes bound. 
The only change to Entropy-SGD is the addition of noise in the outer loop. 
We call the resulting algorithm \ESGLD. (See \cref{entropysgldalg}.)

As we run SGLD longer, we obtain a tighter bound that holds with probability no less than some value approaching $1-\delta$. 
In practice we may not know the rate at which this convergence occurs. 
In our experiments, 
we use very long runs to approximate near-convergence and then only interpret the bounds as being optimistic.
We return to these issues in \cref{sec:eval,discussion}.

\begin{figure*}[ht]
\begin{tikzpicture}[]
  
  \begin{scope}
    \node at (0,2.25) {SGD and SGLD};
    \node at (0,0) {\includegraphics[width=.3\linewidth]{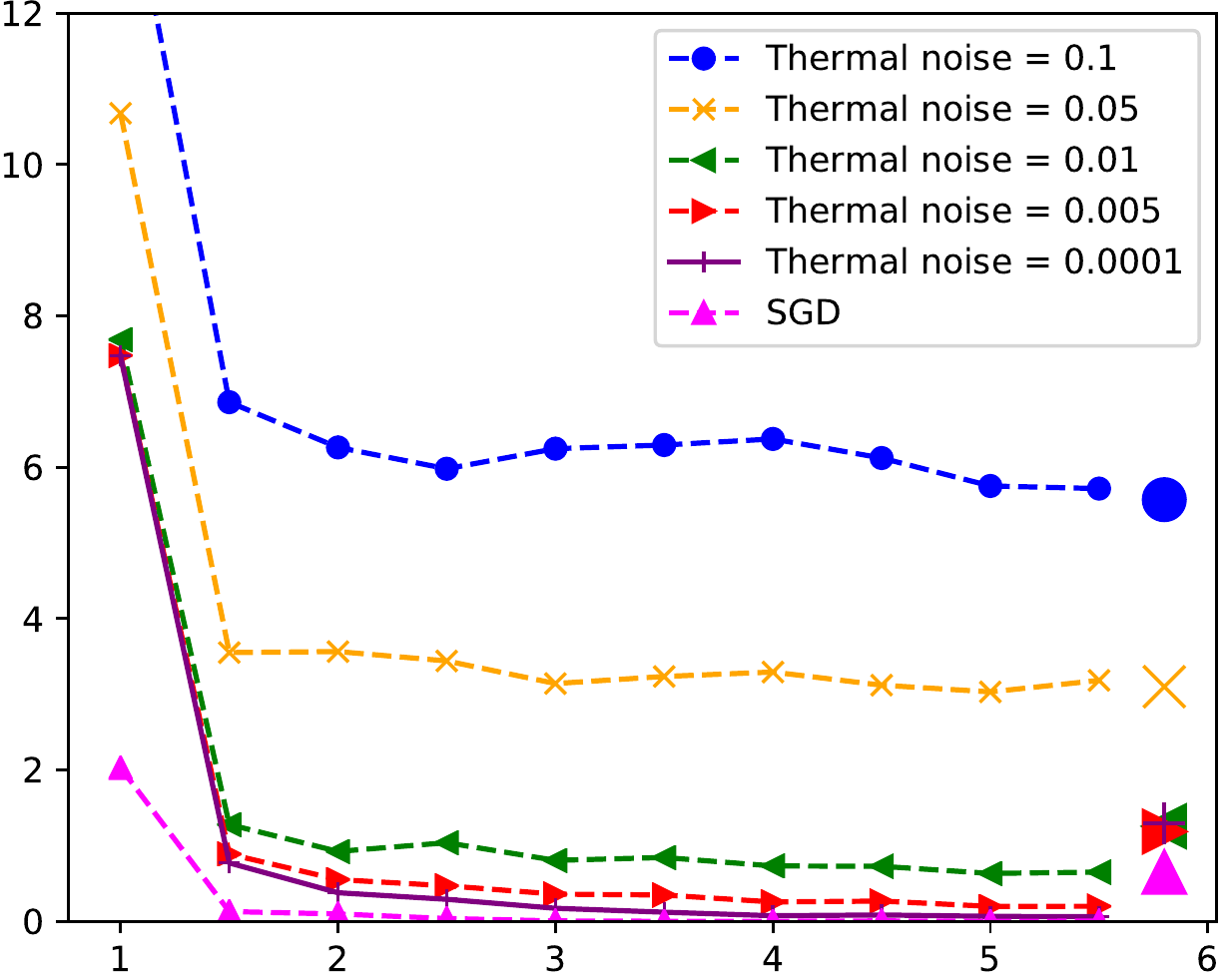}};
    \node[rotate=90] at (-3,0) {\small 0--1 error $\times$ 100};
  \end{scope}
  
  \begin{scope}[xshift=.31\linewidth]
    \node at (0,2.25) {Entropy-SGD};
    \node at (0,0) {\includegraphics[width=.3\linewidth]{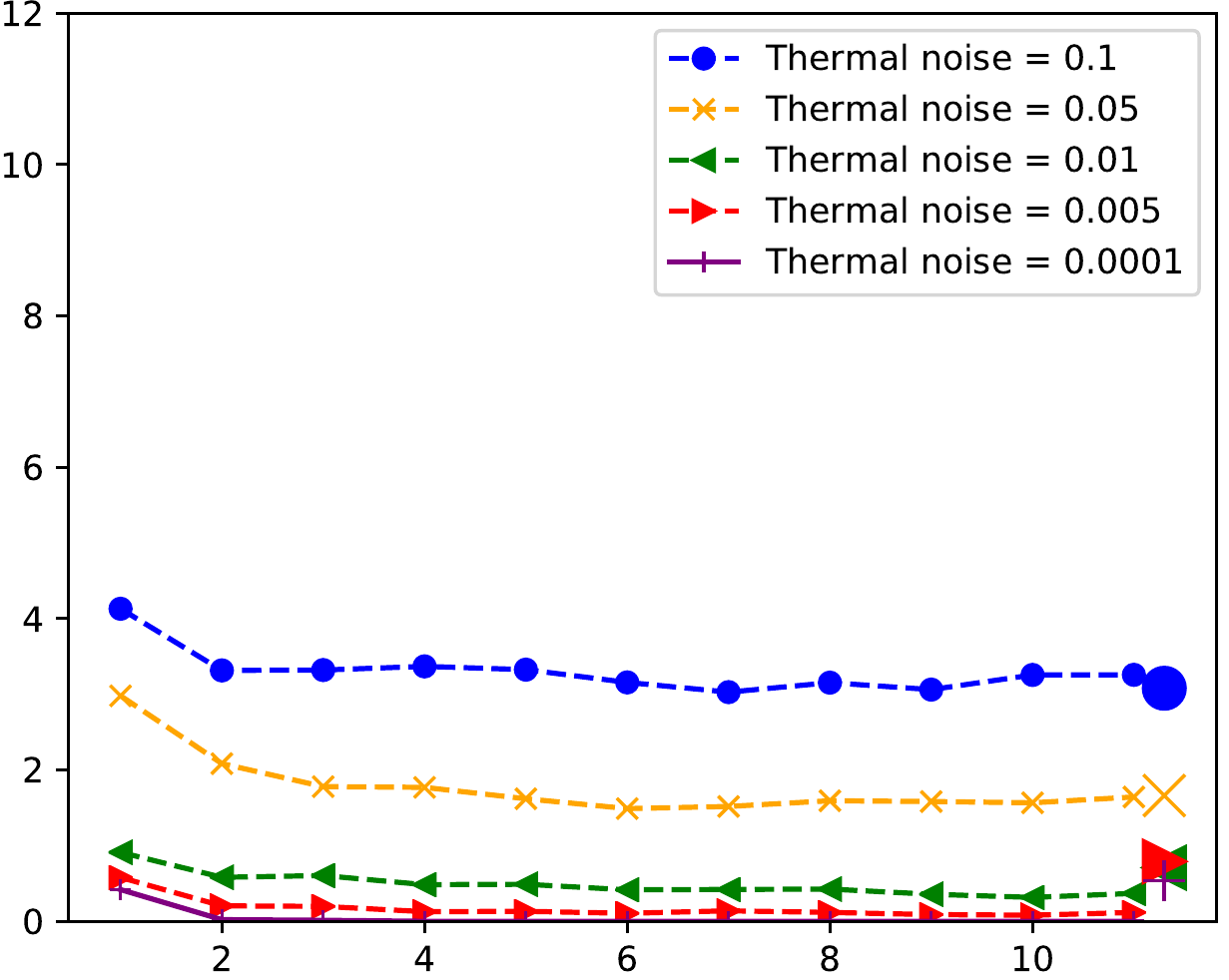}};
  \end{scope}
  
  \begin{scope}[xshift=.62\linewidth]
    \node at (0,2.25) {Entropy-SGLD};
    \node at (0,0) {\includegraphics[width=.3\linewidth]{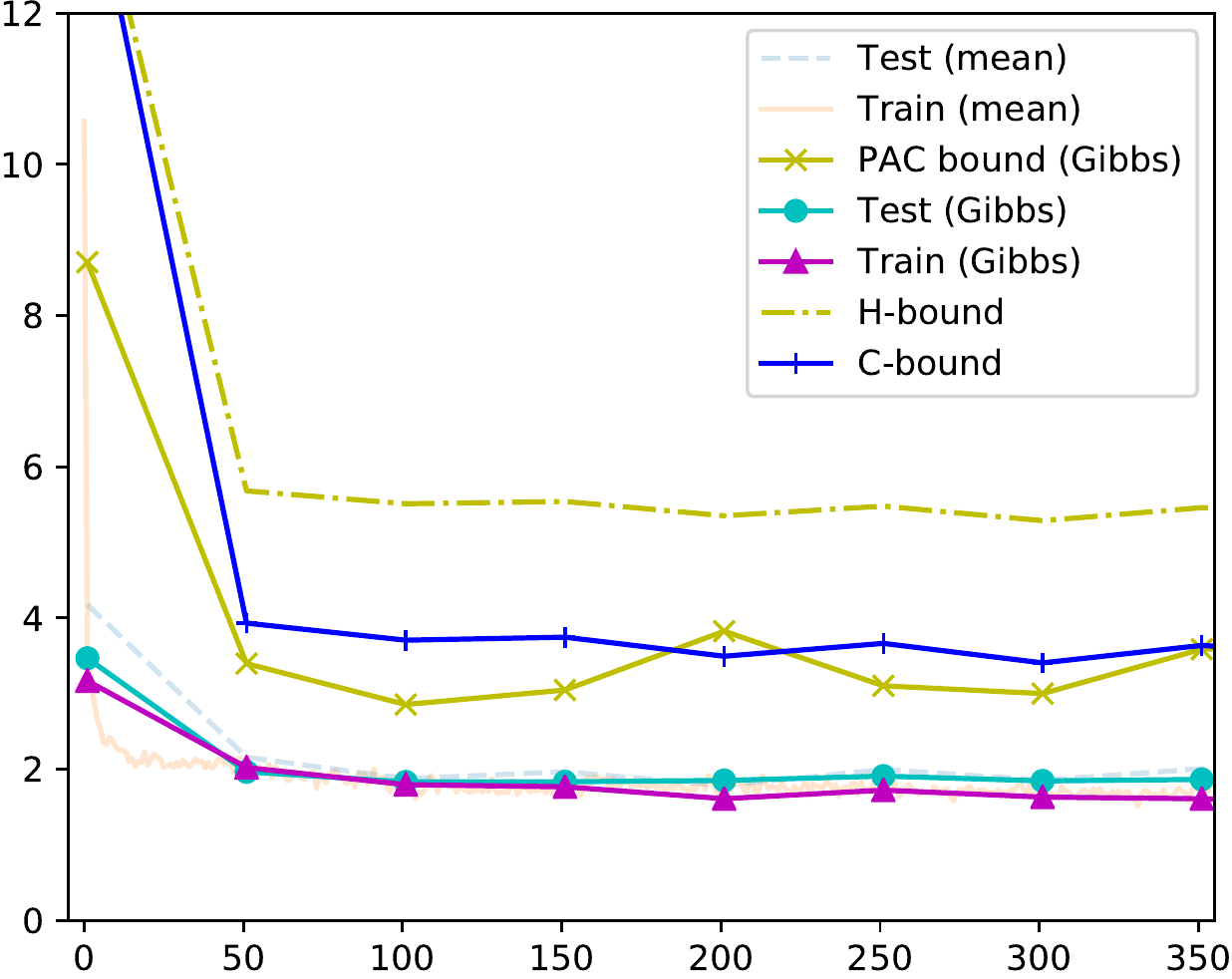}};
    \node[rotate=90] at (2.9,0) {True Labels};
  \end{scope}

  \begin{scope}[yshift=-4.2cm]
  
  \begin{scope}
    \node[rotate=90] at (-3,0) {\small 0--1 error $\times$ 100};
    \node at (0,0) {\includegraphics[width=.3\linewidth]{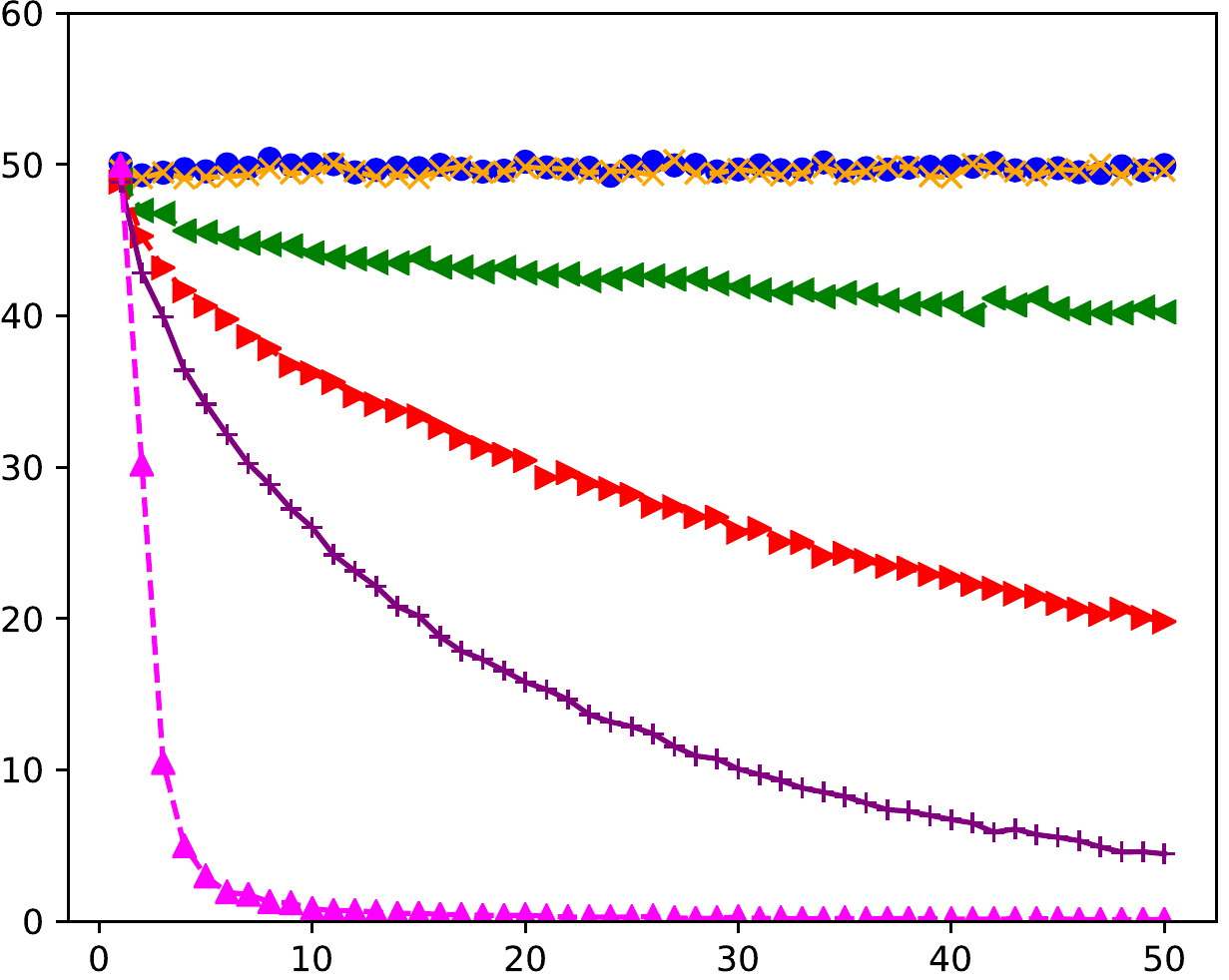}};
    \node at (0,-2.4) {\small \Epochs};
  \end{scope}
  
  \begin{scope}[xshift=.31\linewidth]
    \node at (0,0) {\includegraphics[width=.3\linewidth]{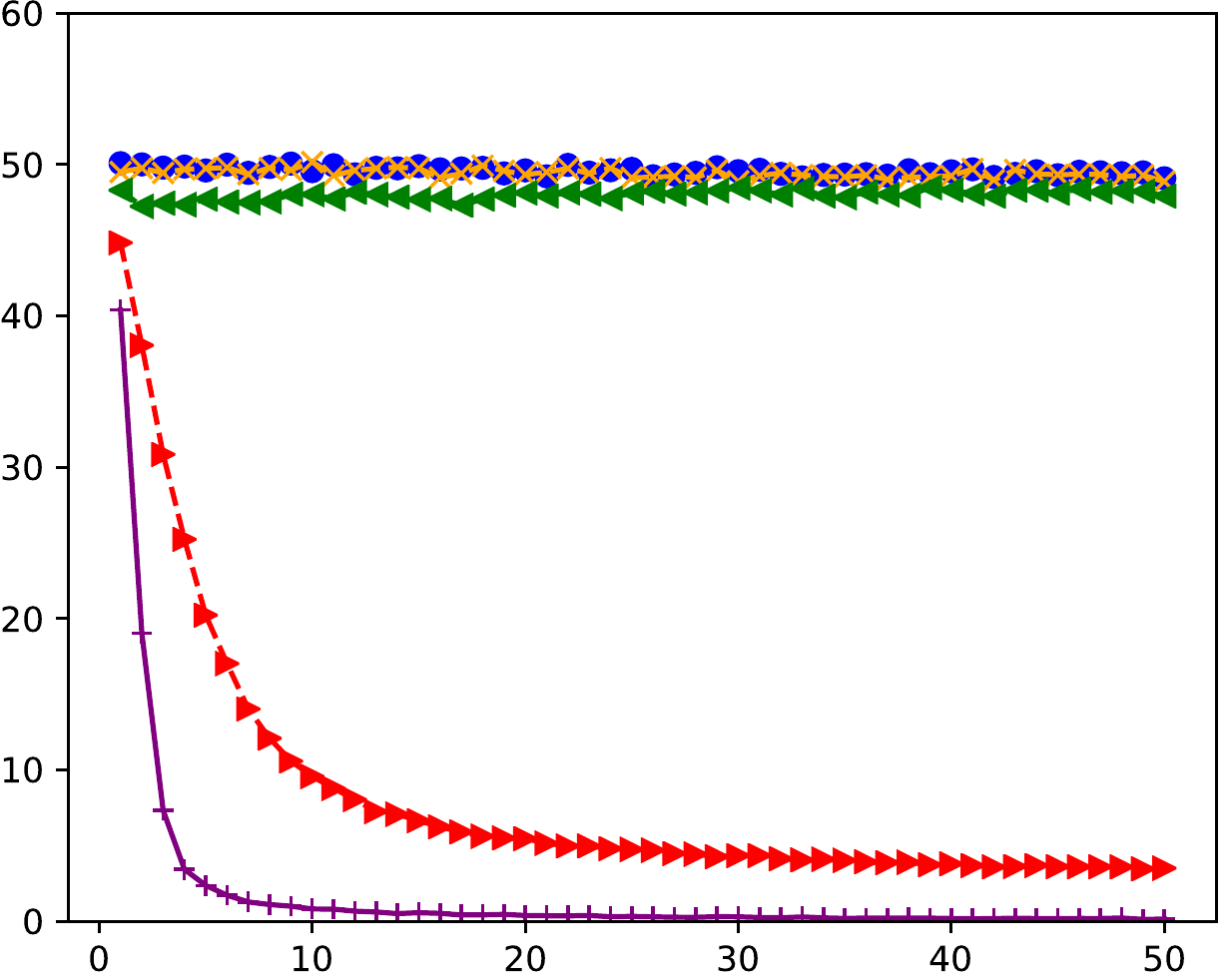}};
    \node at (0,-2.4) {\small \Epochs};
  \end{scope}
   
  \begin{scope}[xshift=.62\linewidth]
    \node at (0,0) {\includegraphics[width=.3\linewidth]{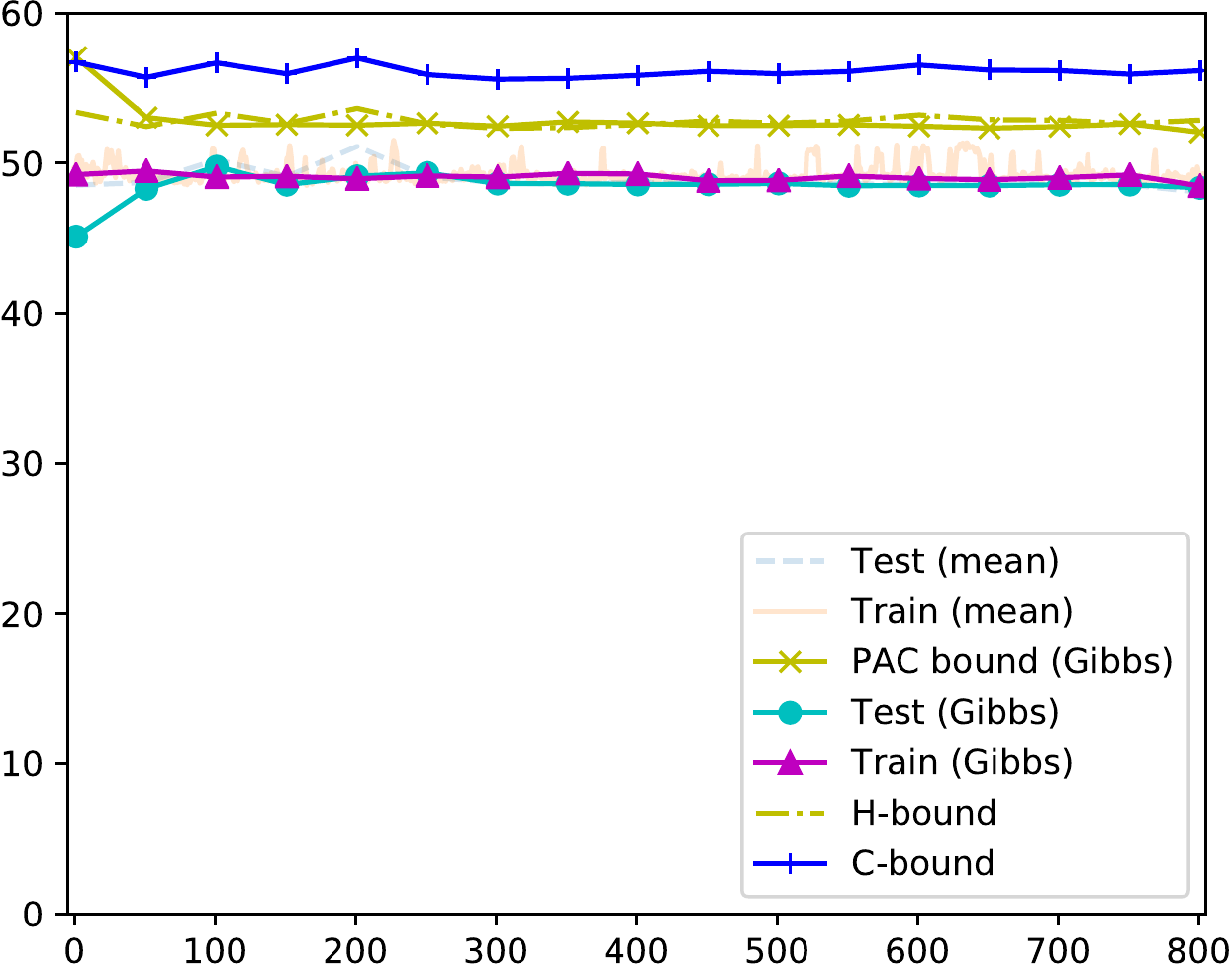}};
    \node[rotate=90] at (2.9,0) {Random Labels};
    \node at (0,-2.4) {\small \Epochs};
  \end{scope}
  
  \end{scope}

\end{tikzpicture}
\caption{
Results on the CONV network on two-class MNIST.
{\bf (left column)} Training error (under 0--1 loss) for SGLD on the empirical risk $-\tau \protect\EmpRisk{S}{}$ under a variety of thermal noise $\sqrt{2/\tau}$ settings. SGD corresponds to zero thermal noise.
{\bf (top-left)} The large markers on the right indicate test error. The gap is an estimate of the generalization error.
On true labels, SGLD finds classifiers with relatively small generalization error. 
At low thermal noise settings, SGLD (and its zero limit, SGD), achieve small empirical risk.
As we increase the thermal noise,
the empirical 0--1 error increases, but the generalization error decreases. At 0.1 thermal noise, risk is close to 50\%.
{\bf (bottom-left)} On random labels, SGLD has high generalization error for thermal noise values 0.01 and below. (True error is $50\%$).  
{\bf (top-middle)} On true labels, Entropy-SGD, like SGD and SGLD, has small generalization error.  
For the same settings of thermal noise, empirical risk is lower.
{\bf (bottom-middle)} On random labels, Entropy-SGD overfits for thermal noise values 0.005 and below. 
Thermal noise 0.01 produces good performance on both true and random labels.
{\bf (right column)} \ESGLD\ is configured to approximately sample from an $\epsilon$-differentially private mechanism with $\epsilon \approx 0.0327$ by setting $\tau = \sqrt{m}$, where $m$ is the number of training samples.
{\bf (top-right)} On true labels, the generalization error for networks learned by \ESGLD\ is close to zero. 
Generalization bounds are relatively tight.
{\bf (bottom-right)} On random label, \ESGLD\ does not overfit. See \cref{sgldconvbounds} for SGLD bounds at same privacy setting.
}
\label{allplots}
\end{figure*}
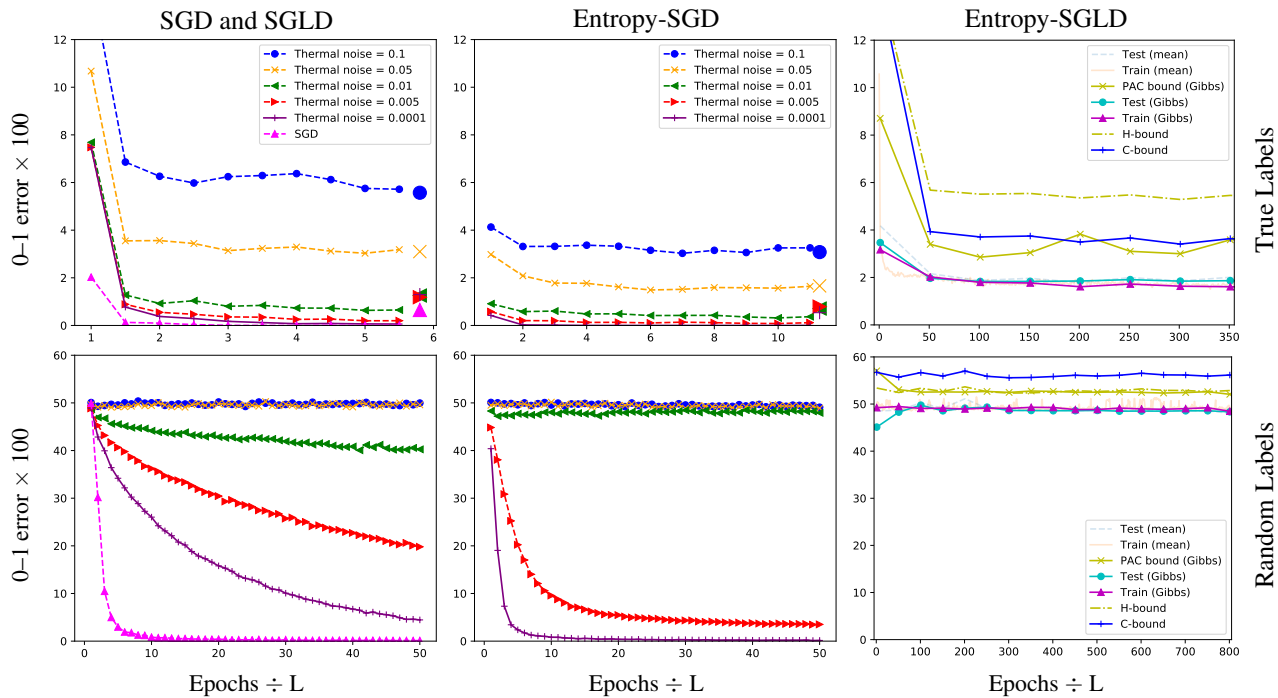

\section{Numerical evaluations on MNIST}
\label{sec:eval}

PAC-Bayes bounds for \ESGLD{} are data-dependent and so the question of their utility is an empirical one that requires data.
In this section, we perform an empirical study of SGD, SGLD, Entropy-SGD, and \ESGLD{} on the MNIST and CIFAR10 data sets, 
using both convolutional and fully connected architectures, and comparing several numerical generalization bounds to test errors estimated based on held-out data. 

The PAC-Bayes bounds we use depend on the privacy of a sample from the local entropy distribution.
(Bounds for SGLD depend on the privacy of a sample from the Gibbs posterior.)
For the local entropy distribution, the degree $\epsilon$ of privacy is determined by the product of the $\tau$ and $\beta$ parameters of the local entropy distribution. (Thermal noise is $\sqrt {2/\tau}$.)
In turn, $\epsilon$ increases the generalization bound.  
For a fixed $\beta$, theory predicts that $\tau$ affects the degree of overfitting. We see this empirically. 
No bound we compute is violated more frequently than it is expected to be.
The PAC-Bayes bound for SGLD is expanded by an amount $\epsilon'$ that goes to zero as SGLD converges. 
We assume SGLD has converged and so the bounds we plot are optimistic.
We discuss this point below in light of our empirical results, and then return to this point in the discussion.

The weights learned by SGD, SGLD, and Entropy-SGD are treated differently from those learned by Entropy-SGLD.
In the former case, the weights parametrize a neural network as usual, and the training and test error are computed using these weights.
In the latter case, the weights are taken to be the mean of a multivariate normal prior, and we evaluate the training and test error of the associated Gibbs posterior (i.e., a randomized classifier).
We also report the performance of the (deterministic) network parametrized by these weights (the ``mean'' classifier) in order to give a coarse statistic summarizing the local empirical risk surface.

Following \citet{Rethinking17}, we study these algorithms on MNIST with the original (``true'') labels, as well as on random labels. 
Parameter $\tau$ that performs very well in one setting often does not perform well in the other. Random labels mimic data where the Bayes error rate is high, and where overfitting can have severe consequences. 

\subsection{Details}

We use a two-class variant of MNIST \citep{MNIST}.\footnote{
The MNIST handwritten digits dataset \citep{MNIST} consists of 60000 training set images and 10000 test set images, labeled 0--9.
We transformed MNIST to a two-class (i.e., binary) classification task by mapping digits 0--4 to label $1$ and 5--9 to label $-1$.
} 
(Due to space issues, see \cref{multiclassmnist,cifar10exp} for experiments on the standard multiclass MNIST dataset and CIFAR10.)
Some experiments involve random labels, i.e., labels drawn independently and uniformly at random at the start of training.
We study three network architectures, abbreviated FC600, FC1200, and CONV. 
Both FC600 and FC1200 are 3-layer fully connected networks, 
with 600 and 1200 units per hidden layer, respectively.
CONV is a convolutional architecture.
All three network architectures are taken from the MNIST experiments by \citet{CCSL16}, 
but adapted to our two-class version of MNIST.\footnote{
We adapt the code provided by \citeauthor{CCSL16}, with some modifications to the training procedure
and straightforward changes necessary for our binary classification task.
}
Let $S$ and $\Stest$ denote the training and test sets, respectively.
For all learning algorithms we track 
\setlist{nolistsep}
\begin{enumerate}[label=(\roman*),noitemsep]
\item $\EmpErr{S}{\ww}$ and $\EmpErr{\Stest}{\ww}$, i.e., the training/test error for $\ww$.
\end{enumerate}
We also track
\begin{enumerate}[resume*]
\item estimates of $\EmpErr{S}{\LGibbs}$ and $\EmpErr{\Stest}{\LGibbs}$, i.e.,
the mean training and test error of the local Gibbs distribution, viewed as a randomized classifier (``Gibbs'')
\end{enumerate}
and, using the bound stated in \cref{dpGibbspSample}, we compute
\begin{enumerate}[resume*,noitemsep]
\item a PAC-Bayes bound on $\Err{\Dist}{\LGibbs}$ using \cref{DPpacbayes} (``PAC-bound''); 
\item the mean of a Hoeffding-style bound on $\ErrRamp{\Dist}{\ww'}$, where the underlying loss is the ramp loss with slope $10^6$ and $\ww' \sim \LEDist{\!}$,
using the first bound of \cref{onetobounds} (``H-bound''); 
\item an upper bound on the mean of a Chernoff-style bound on $\ErrRamp{\Dist}{\ww'}$, 
where $\ww' \sim \LEDist{\!}$,
using the second bound of \cref{onetobounds} (``C-bound'').
\end{enumerate}
We also compute H- and C- bounds for SGLD, viewed as a sampler for $\ww' \sim P_{\exp\parens{-\tau \EmpRisk{S}{}}}$, where $P$ is Lebesgue measure.

In order to get privacy guarantees for SGLD and \ESGLD, we modify the cross entropy loss function to be bounded following \citet{DR18private}. (See \cref{binobjective}). 
With the choice of $\beta=1$ and $\tau = \sqrt{m}$, and the loss function taking values in an interval of length $\Lmax = 4$, 
the local entropy distribution is an $\epsilon$-differentially private mechanism with $\epsilon \approx 0.0327$. 
See \cref{apphyper} for additional details.
Note that, in the calculation of (iii), we do not account for Monte Carlo error in our estimate of $\EmpErr{S}{\ww}$. 
The effect is small, given the large number of iterations of SGLD performed for each point in the plot.
Recall that
\begin{equation*}
\Err{\Dist}{\LGibbs} = \EEE{\ww' \sim \smash{\LGibbs}} \event { \Err{\Dist}{\ww'}},
\end{equation*}
and so we may interpret the bounds in terms of the performance of a randomized classifier or the mean performance of a randomly chosen classifier.

\subsection{Results}

Key results for the convolutional architecture (CONV) appear in \cref{allplots}. 
Results for FC600 and FC1200 appear in \cref{fcplots} of \cref{expdetails}. (Training the CONV network produces the lowest training/test errors and tightest generalization bounds.
Results and bounds for FC600 are nearly identical to those for FC1200, despite FC1200 having three times as many parameters.)

The left column of \cref{allplots} presents the performance of SGLD for various levels of thermal noise $\sqrt{2/\tau}$ under both true and random labels. 
(Assuming SGLD is close to weak convergence, 
we may also use SGLD to directly perform a private optimization of the empirical risk surface.
The level of thermal noise determines the differential privacy of SGLD's stationary distribution and so we expect to see a tradeoff between empirical risk and generalization error. 
Note that, algorithmically, SGD is SGLD with zero thermal noise.)
SGD achieves the smallest training and test error on true labels, but overfits the worst on random labels.
In comparison, SGLD's generalization performance improves with higher thermal noise, while its risk performance worsens. 
At 0.05 thermal noise, SGLD achieves reasonable but relatively large risk but almost zero generalization error on both true and random labels. Other thermal noise settings have either much worse risk or generalization performance.

The middle column of \cref{allplots} presents the performance of Entropy-SGD for various levels of thermal noise $\sqrt{2/\tau}$ under both true and random labels. 
As with SGD, Entropy-SGD's generalization performance improves with higher thermal noise, while its risk performance worsens. 
At the same levels of thermal noise, Entropy-SGD outperforms the risk and generalization error of SGD.
At 0.01 thermal noise, Entropy-SGD achieves good risk and low generalization error on both true and random labels.
However, the test-set performance of Entropy-SGD at 0.01 thermal noise is still worse than that of SGD.
Whether this difference is due to SGD overfitting to the MNIST test set is unclear and deserves further study.

The right column of \cref{allplots} presents the performance of Entropy-SGLD with $\tau = \sqrt{m}$ on true and random labels. (This corresponds to approximately 0.09 thermal noise.)
On true labels, both the mean and Gibbs classifier learned by Entropy-SGLD have approximately 2\% test error and essentially zero generalization error, which is less than predicted by the bounds evaluated.
The differentially private PAC-Bayes risk bounds are roughly 3\%.
As expected by the theory, Entropy-SGLD, properly tuned, does not overfit on random labels, even after thousands of epochs.

We find that the PAC-Bayes bounds are generally tighter than the H- and C-bounds. 
All bounds are nonvacuous, though still loose.
The error bounds reported here are tighter than those reported by \citet{DR17}.
However, 
\emph{the bounds are optimistic because they do not include the additional term 
which measure how far SGLD is from its weak limit.}
Despite the bounds being optimistically tight, we see almost no violations in the data. (Many violations would undermine our assumption.)
While we observe tighter generalization bounds than previously reported, and better test error, 
we are still far from the performance of SGD. 
The optimistic picture we get from the bounds suggests we need to develop new approaches.
Weaker notions of stability with respect to the training data/privacy may be necessary to achieve further improvement in generalization error and test error.

\section{Discussion}
\label{discussion}

Our work reveals that Entropy-SGD can be understood as optimizing a PAC-Bayes generalization bound in terms of the bound's prior. 
Because the prior must be independent of the data, the bound is invalid, 
and, indeed, we observe overfitting in our experiments with Entropy-SGD 
when the thermal noise $\sqrt{2/\tau}$ is set to 0.0001 as suggested by \citeauthor{CCSL16} for MNIST.

PAC-Bayes priors can, however, depend on the data distribution. 
This flexibility seems wasted, since the data sample is typically viewed as one's only view onto the data distribution.
However, using results combining differential privacy and PAC-Bayes bounds, 
we arrive at an algorithm, \ESGLD, that minimizes its own PAC-Bayes bound (though for a surrogate risk).
\ESGLD{} performs an approximately private computation on the data, extracting information about the underlying distribution, 
without undermining the statistical validity of its PAC-Bayes bound.
The cost of using the data is a looser bound, but the gains in choosing a better prior make up for the loss.
(The gains come from the KL term being much smaller on the account of the prior being better matched to the data-dependent posterior.)

Our bounds based on \cref{maindppacthm} are optimistic because we do not include the $\epsilon'$ term, assuming that SGLD has essentially converged. 
We do not find overt evidence that our approximation is grossly violated, which would be the case if
we saw the test error repeatedly falling outside our confidence intervals.
We believe that it is useful to view the bounds we obtain for \ESGLD{}
as being optimistic and representing the bounds we might be able to achieve
rigorously should there be a major advance in private optimization. 
(No analysis of the privacy of SGLD takes advantage of the fact that it mixes weakly, in part because it's difficult to characterize how much it has converged in any real-world setting after a finite number of steps.) 
On the account of using private data-dependent priors (and making optimistic assumptions), 
the bounds we observe for \ESGLD{} 
are significantly tighter than those reported by \citet{DR17}.
However, despite our bounds potentially being optimistic, 
the test set error we are able to achieve is still 5--10 times worse than that of SGD.
Differential privacy may be too conservative for our purposes,
leading us to underfit. 
We are able to achieve good generalization on both true and random labels under 0.01 thermal noise,
despite this value of noise being too large for tight bounds.
Identifying the appropriate notion of privacy/stability to combine with PAC-Bayes bounds is an important problem.

Despite \ESGLD\ having much stronger generalization guarantees,
\ESGLD\ learns much more slowly than Entropy-SGD, 
the test error of \ESGLD\ is far from state of the art, 
and the  PAC-Bayes bounds, while much tighter than existing bounds, are still quite loose.
It seems possible that we may be facing a fundamental tradeoff between the speed of learning, 
the excess risk, 
and the ability to produce a certificate of one's generalization error via a rigorous bound.
Characterizing the relationship between these quantities is an important open problem.

\paragraph{Acknowledgments}

This research was carried out in part while the authors were visiting the Simons Institute for the Theory of Computing at UC Berkeley.
The authors would like to thank 
Pratik Chaudhari,
Pascal Germain,
David McAllester,
and
Stefano Soatto
for helpful discussions.
GKD is supported by an EPSRC studentship.  
DMR is supported by an
NSERC Discovery Grant, Connaught Award, Ontario Early Researcher Award, and U.S. Air Force Office of Scientific
Research grant \#FA9550-15-1-0074.

\printbibliography

\clearpage
\appendix

\section{Maximizing local entropy minimizes a PAC-Bayes bound}
\label{localpacbayes}

\begin{proof}[Proof of \cref{optprior}]
Let $m$, $\delta$, $\Dist$, and $P$ be as in \cref{pacbayeslinear}
and let $S \sim \Dist^m$. 
The linear PAC-Bayes bound (\cref{pacbayeslinear})
ensures that for any fixed $\lambda>1/2$ and bounded loss function, 
with probability at least $1-\delta$ over the choice of $S$,
the bound
\begin{equation*}
\left( 1 - \frac{1}{2\lambda} \right)  \frac{m\,\Risk{\Dist}{Q}}{\lambda \Lmax}  
\le \frac{m\, \EmpRisk{S}{Q}}{\lambda \Lmax} + \KL{Q}{P} + g(\delta).
\end{equation*}
holds for all $Q \in \ProbMeasures{\HH}$.
Minimizing the upper bound on 
 $\Risk{\Dist}{Q}$ 
is equivalent to the problem
\begin{equation}\label{pbinf}
\textstyle \inf_{Q}   \  Q[\rr] + \KL{Q}{P}
\end{equation}
with $\rr (h) = \frac{m}{\lambda \Lmax}\,\EmpRisk{S}{h}$. 
By \citep[][Lem.~1.1.3]{Catoni}, for all 
 $Q\in \ProbMeasures{\HH}$ with $\KL{Q}{P} < \infty$,
\begin{equation}\label{eqcatoni}
- \log P[\exp(-\rr)] = Q[\rr] + \KL{Q}{P} - \KL{Q}{P_{\exp(-\rr)}}.
\end{equation}
Using \cref{eqcatoni}, we may reexpress \cref{pbinf} as
\begin{equation*}
\textstyle \inf_{Q}  \ \KL{Q}{P_{\exp(-\rr)}} - \log P[\exp(-\rr)] .
\end{equation*}
By the nonnegativity of the Kullback--Liebler divergence, 
the infimum is achieved when the KL term is zero, 
i.e., when $Q = P_{\exp(-\rr)}$.
Then
\begin{equation*}
\left( 1- \frac{1}{2 \lambda}  \right) \frac{m}{\lambda \Lmax} \Risk{\Dist}{P_{\exp(-\rr)}}
\le - \log P[\exp(-\rr)] + g(\delta).
\end{equation*}
Finally, it is plain to see that $F_{\gamma,\tau}(\ww;S) = C + \log P[\exp(-\rr)]$
when $C = \frac 1 2 p \log \parens {2\pi (\tau\gamma)^{-1} }$ is a constant, $\tau = \frac{m}{\lambda \Lmax}$, and $P = \Normal(\ww, (\tau\gamma)^{-1} I_{p})$ is a multivariate normal with mean $\ww$ 
and covariance matrix $(\tau\gamma)^{-1} I$. 
\end{proof}

\section{Privacy of local entropy distribution}
\label{localentdist}
\begin{proof}[Proof of \cref{dpGibbspSample}]
The result follows immediately from the following two lemmas.
\end{proof}

\begin{lemma}[{\citealt[Thm.~6]{ExpRelease07}}]\label{exprelease}
Let $ q : Z^m \times \HH \to \Reals$ be measurable, let $P$ be a measure on $\HH$,
let $\beta > 0$,
and assume $P[\exp \parens {-\beta\, q(S,\cdot)}] < \infty$ for all $S \in Z^m$.
Let $\Delta q \defas \sup_{S,S'} \sup_{\ww \in \HH} | q(S,\ww) - q(S',\ww) |$, 
where the first supremum ranges over pairs $S,S' \in Z^m$ that disagree on no more than one coordinate.
Let $\Alg : Z^m \randto \HH$, on input $S \in Z^m$,
output a sample from the Gibbs distribution $P_{\exp \parens {-\beta q(S,\cdot)}}$.
Then $\Alg$ is $2\beta \Delta q$-differentially private.
\end{lemma}

\begin{lemma}\label{boundedlemma}
Let $F_{\gamma,\tau}(\ww; S)$ be defined as \cref{localentdef},
assume the range of the loss is contained in an interval of length $\Lmax$,
and define $q(S,\ww) = - F_{\gamma,\tau}(\ww;S)$.
Then 
$\Delta q \defas \sup_{S,S'} \sup_{\ww \in \HH} | q(S,h) - q(S',h) | \le \frac{\Lmax \tau}{m}$. 
\end{lemma}
\begin{proof}
The proof essentially mirrors that of \citep[Thm.~6]{ExpRelease07}.
\end{proof}

\section{Two-class MNIST experiments}
\label{expdetails}

\subsection{Architecture} 
\label{architecture}

We study three architectures: CONV, FC600, and FC1200.

CONV is a convolutional neural network, 
whose architecture is the same as that used by \citet{CCSL16} for multiclass MNIST classification, except modified to produce a single probability output for our two-class variant of MNIST.
In particular, CONV has two convolutional layers, a fully connected ReLU layer, and a sigmoidal output layer, 
yielding $126,711$ parameters in total.

FC600 and FC1200 are fully connected 3-layer neural networks,
with 600 and 1200 hidden units, respectively,
yielding $834,601$ and $2,385,185$ parameters in total, respectively.
We use ReLU activations for all but the last layer, which was sigmoidal to produce an output in $[0,1]$.

In their MNIST experiments, \citet{CCSL16} use dropout and batch normalization. 
We do not use dropout.
The bounds we achieve with and without batch norm are very similar. 
Without batch norm, however, it is necessary to tune the learning rates.
Understanding the combination of SGLD and batch norm and the limiting invariant distribution, if any, is an important open problem.

\subsection{Training objective and hyperparameters for \ESGLD} 
\label{apphyper}

\begin{figure*}[ht]
\centering
\begin{tikzpicture}[]
  
  \begin{scope}[xshift=.1\linewidth]
  \begin{scope}
    \node at (0,2.8) {FC600};
    \node[rotate=90] at (-3.5,0) {\small 0--1 error $\times$ 100};
    \node at (0,-2.9) {\small \Epochs};
    \node at (0,0) {\includegraphics[width=.38\linewidth]{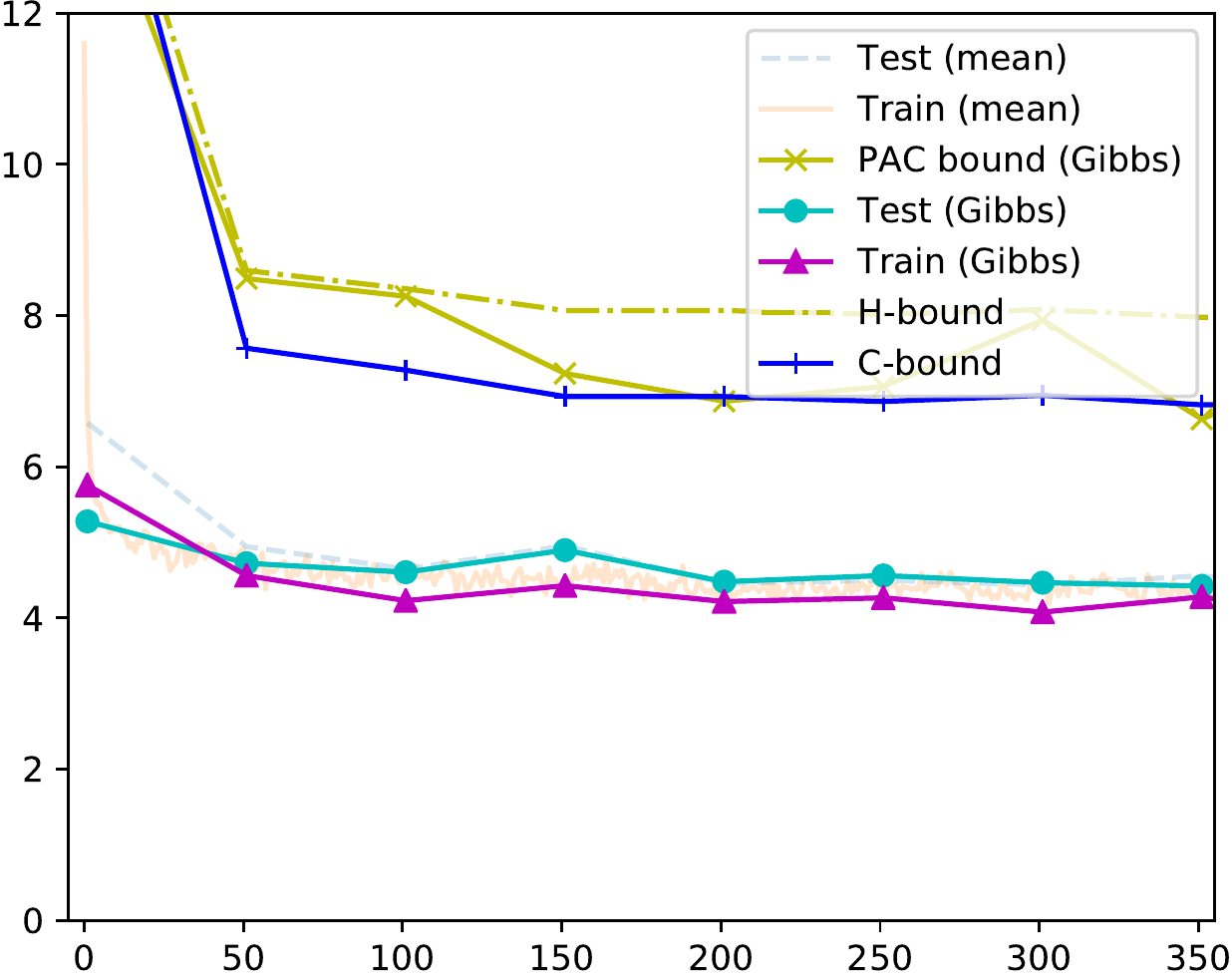}};
  \end{scope}
  \begin{scope}[xshift=.47\linewidth]
    \node at (0,0) {\includegraphics[width=.38\linewidth]{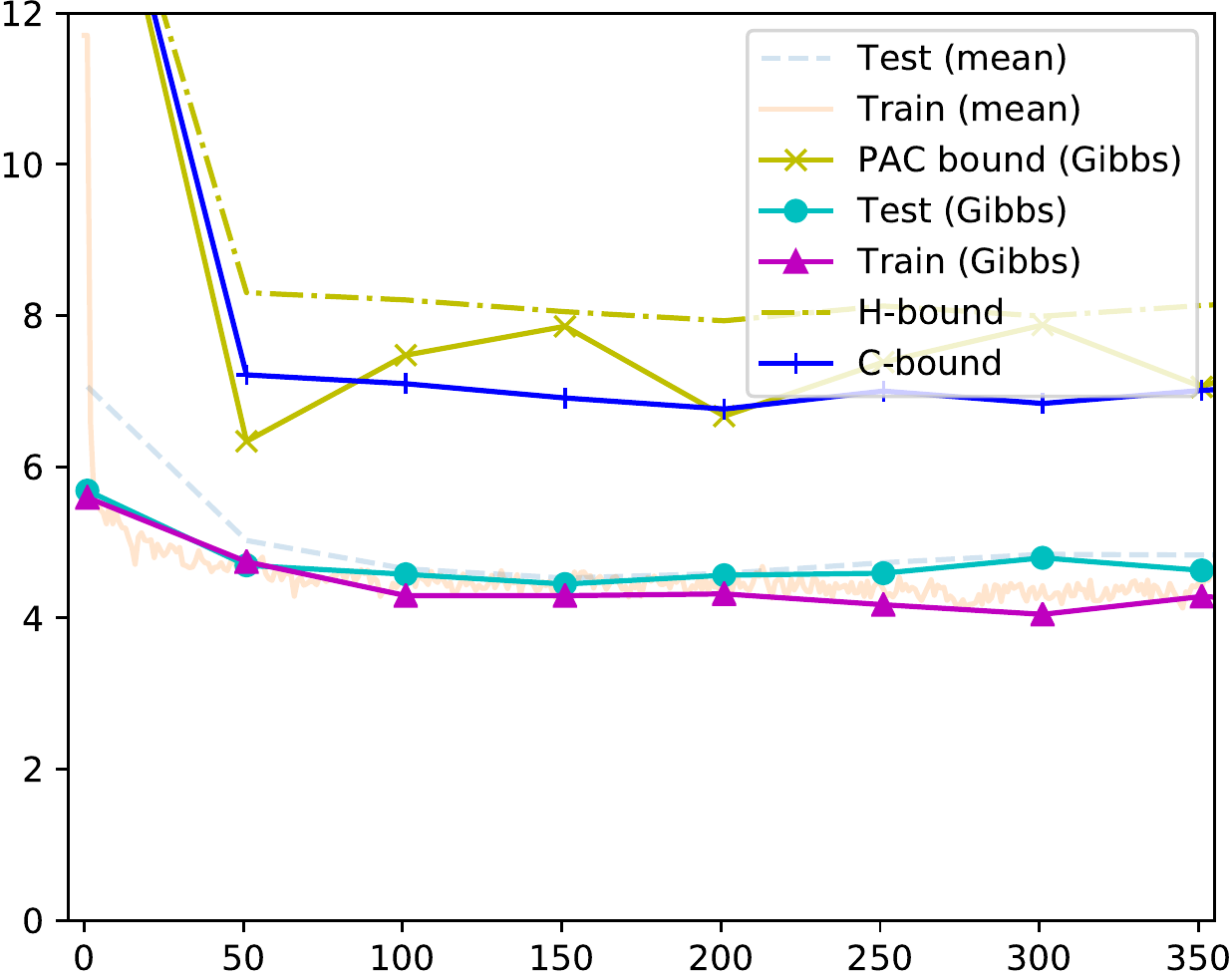}};
    \node at (0,2.8) {FC1200};
    \node[rotate=-90] at (3.5,0) {True Labels};
    \node at (0,-2.9) {\small \Epochs};
  \end{scope}
  \end{scope}
\end{tikzpicture}
\caption{
Fully connected networks trained on binarized MNIST with a differentially private \ESGLD\ algorithm.
{\bf (left)} \ESGLD\ applied to FC600 network trained on true labels. 
{\bf (right)} \ESGLD\ applied to FC1200 network trained on true labels.
Both training error and generalization error are similar for both network architectures. 
The true generalization gap is close to zero, since the test and train error overlaps. 
All the computed bounds on the test error are loose but nonvacuous.}
\label{fcplots}
\begin{tikzpicture}[]
  
  \begin{scope}[xshift=.1\linewidth]
  \begin{scope}
    \node at (0,0) {\includegraphics[width=.38\linewidth]{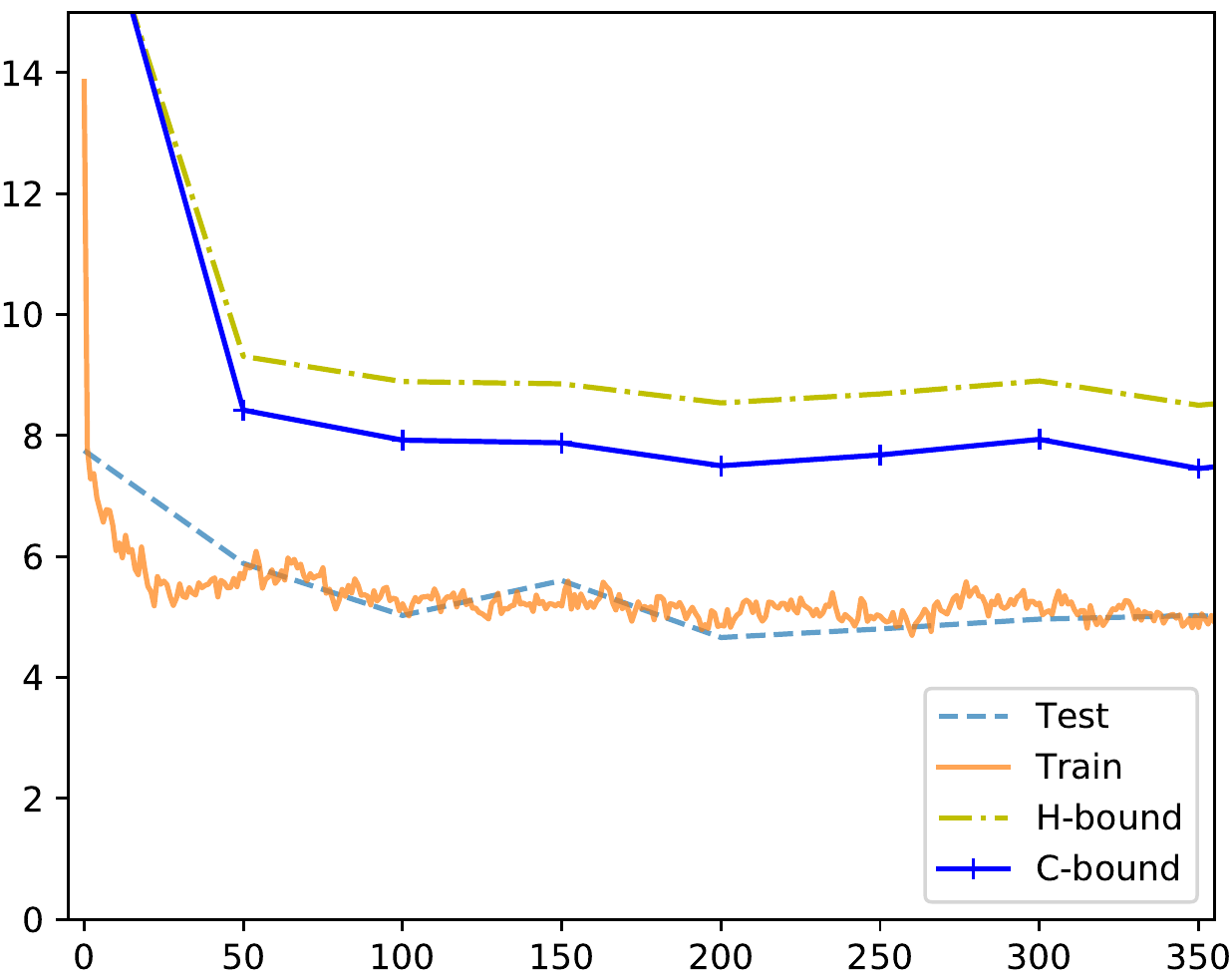}};
    \node[rotate=90] at (-3.5,0) {\small 0--1 error $\times$ 100};
    \node at (0,-2.9) {\small Epochs};
  \end{scope}
  \begin{scope}[xshift=.47\linewidth]
    \node at (0,0) {\includegraphics[width=.38\linewidth]{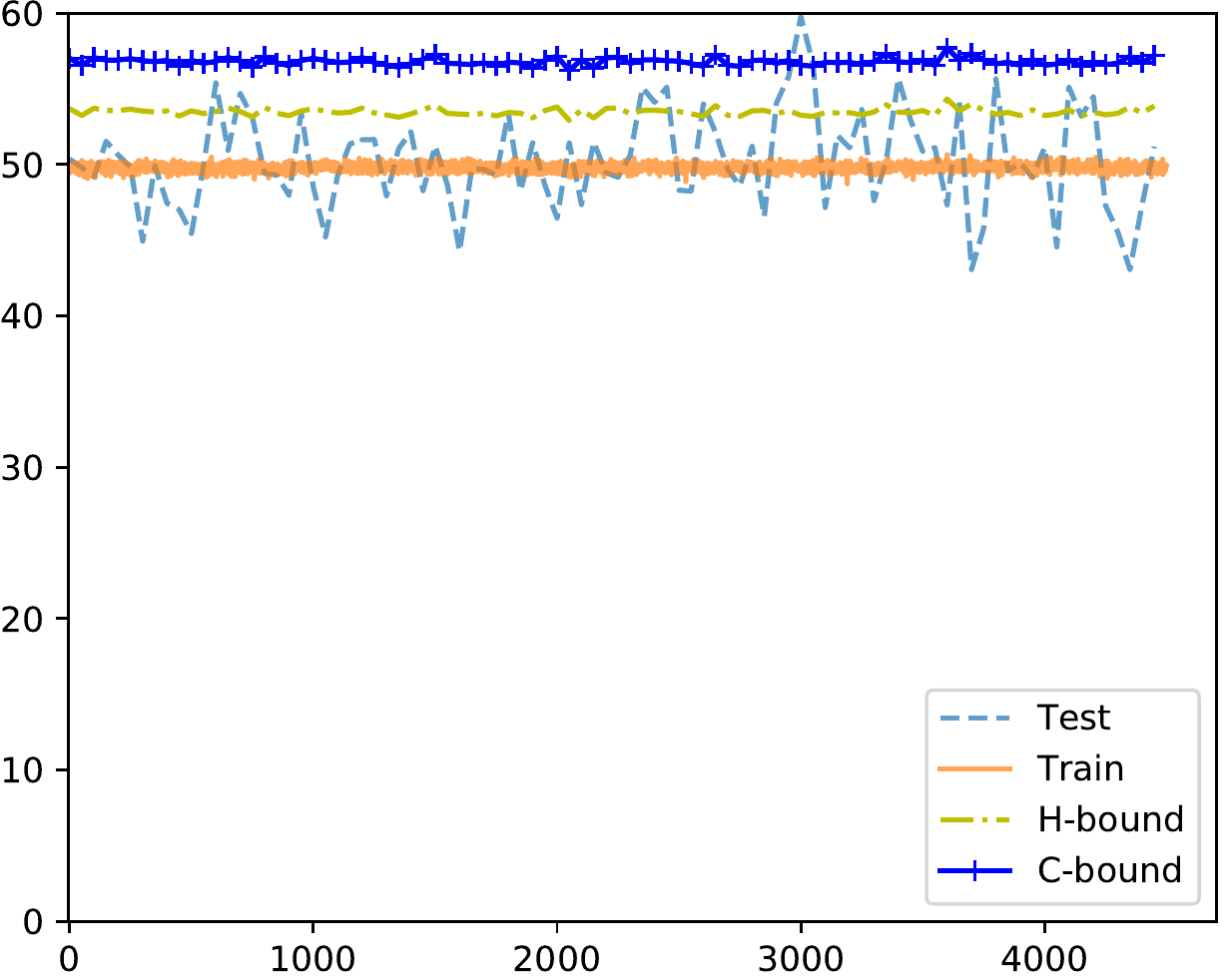}};
    \node[rotate=-90] at (3.5,0) {SGLD};
    \node at (0,-2.9) {\small Epochs};
  \end{scope}
  \end{scope}  
\end{tikzpicture}
\caption{
Results on CONV architecture, 
running SGLD configured to have the same differential privacy as \ESGLD\ with $\tau = \sqrt{m}$.
On true labels, SGLD learns a network with approximately $3\%$ higher training and test error than the mean and Gibbs networks learned by \ESGLD.
SGLD does not overfit on random labels, as predicted by theory.
The C-bound on the true error of this network is around $8\%$, which is worse than the roughly $4\%$ C-bound on the mean classifier. 
}
\label{sgldconvbounds}
\end{figure*}

\subsubsection{Objective} \label{binobjective}

All networks are trained to minimize a bounded variant of empirical cross entropy loss.
In the binary classification setting,
the output of the network is an element $p\in[0,1]$, i.e., the probability the label is one.
The binary cross entropy, $\lossbce$, 
is then $-y \log(p) - (1-y) \log(1-p) $.

\citet{DR18private} make the following change to produce a bounded loss involves replacing 
$ -y \log(p) - (1-y) \log(1-p) $
with
$-y \log \psi(p) - (1-y) \log(1-\psi(p)) $, where
\begin{equation}\label{affine}
\psi(p) = e^{-\Lmax} + ( 1- 2 \e^{-\Lmax}) p
\end{equation}
is an affine transformation that maps $[0,1]$ to $[e^{-\Lmax},1-e^{-\Lmax}]$, thereby removing extreme probability values.
As a result, the binary cross entropy loss $\lossbce$ is contained in the interval $[0,\Lmax]$.
We take $\Lmax = 4$ in our experiments. 

\subsubsection{Epochs}

Ordinarily, an epoch implies one pass through the entire data set.
For SGD, each stochastic gradient step processes a minibatch of size $K=128$.
Therefore, an epoch is $m/K=468$ steps of SGD.
An epoch for Entropy-SGD and Entropy-SGLD is defined as follows: each iteration of the inner SGLD loop processes a minibatch of size $K=128$, and the inner loop runs for $L=20$ steps.
Therefore, an epoch is $m/(LK)$ steps of the outer loop.
In concrete terms, there are $20$ steps of SGD per every one step of Entropy-SG(L)D.
Concretely, the x-axis of our plots measure epochs divided by $L$. This choice, used also by \citet{CCSL16}, ensures that the wall-clock time of Entropy-SG(L)D and SGD align.

\subsubsection{SGLD parameters: step sizes and weighted averages}

The step sizes for SGLD must be square summable but not summable.
The step sizes for the outer SGLD loop are of the form $\eta_t = \eta t^{-0.6}$, with $\eta = \frac{\eta' }{ \gamma \tau}$, where $\eta' = 0.006$ and is called the base learning rate.
The step sizes for the inner SGLD loop are of the form $\eta_t = \eta t^{-1}$, with $\eta = \frac{2}{\tau}$.

The estimate produced by the inner SGLD loop is computed using a weighted average (line 8) with $\alpha = 0.75$.
We use SGLD again when computing the PAC-Bayes generalization bound (\cref{estKL}).  
In this case, SGLD is used to sample from the local Gibbs distribution when estimating the Gibbs risk and the KL term. We run SGLD for 1000 epochs to obtain our estimate.
Again, we use weighted averages, but with $\alpha = 0.005$, in order to average over a larger number of samples and better control the variance.

\subsubsection{Gibbs classifier parameters}

We set $\gamma = 1$ and $\tau = \sqrt{m}$  and keep the values fixed during optimization. 
By \cref{dpGibbspSample}, the value of $\tau$, $\Lmax$, and $\beta$ determine the differential privacy of sampling once from the local entropy distribution, which in turn affects the PAC-Bayes bounds for \ESGLD. 
The differential privacy parameter $\epsilon$ and confidence parameter $\delta$ contribute
\[
 2 \frac{\max \{
                 \ln \frac {3}{\delta}, \ 
                 m \epsilon^2 
                       \} }{m} 
\]
to the bound on the KL-generalization error $\KLbin{\EmpErr{S}{Q}}{\Err{\Dist}{Q}}$
in the differentially private PAC-Bayes bound (\cref{maindppacthm}).
Choosing $\tau = \sqrt{m}$, implies that the contribution coming from differential privacy decays at a rate of $1/m$.
Numerically, given $\Lmax = 4$ and $\beta = 1$, this contribution is $0.002$.

\subsection{Evaluating the PAC-Bayes bound}

\subsubsection{Inverting $\KLbin{q}{p}$}

In order to bound the risk using the differentially private PAC-Bayes bound,
we must compute the largest value $p$ such that $\KLbin{q}{p} \le c$. 
There does not appear to be a simple formula for this value.
In practice, however, the value can be efficiently numerically approximated using, e.g., Newton's method. 
See \citep[\S2.2~and~App.~B]{DR17}.

\subsubsection{Estimating the KL divergence}
\label{estKL}

Let $\ell(\ww) = \tau\,\SurEmpRisk{S}{\ww}$.
By \citep[][Lem.~1.1.3]{Catoni}, 
\[
\KL{P_{\exp(-\ell)}}{P}
= \EEE{\ww \sim P_{\exp(-\ell)}} \!\!\!\!\!\brackets { - \ell (\ww) }
        - \log P[\exp(-\ell)].
\]
\citet{DR18private} make use of this to propose the two following Monte Carlo estimates:
\[\label{elossundergibbs}
\EEE{\ww \sim P_{\exp(-\ell)}} \!\!\!\!\!\brackets { - \ell(\ww) }
\approx -  \frac 1 {k'} \sum_{i=1}^{k'} \ell(\ww')
\]
where $\ww'_1,\dots,\ww'_{k'}$ are taken from a Markov chain targeting $P_{\exp(-\ell)}$, 
such as SGLD run for $k'\gg1$ steps (which is how we computed our bounds),
and
\[
\log P[\exp(-\ell)] 
&= \log \int \exp \{ - \ell(\ww) \} \, P(\dee \ww) \\
&\gtrapprox \log \frac 1 k \sum_{i=1}^k \exp \{ - \ell(\ww_i) \}.
\]
where $h_1,\dots,h_k$ are i.i.d.\ $P$ (which is a multivariate Gaussian in this case).
In the latter case, due to the concavity of $\log$, the estimate is a lower bound with high probability, yielding a high probability upper bound on the KL term.

\section{Multiclass MNIST experiments} 
\label{multiclassmnist}

We evaluate the same generalization bounds on the standard MNIST classification task as in the MNIST binary labelling case. 
The results are presented in \cref{multiclass}.

All the details of the network architectures and parameters are as stated in \cref{apphyper}, with two exception: 
following \citet{CCSL16}, we use a fully connected network with 1024 hidden units per layer, denoted FC1024.

\subsection{Objective}
\label{multiclassobjective}

The neural network produces a probability vector $(p_1,\dots,p_K)$ via a soft-max operation. 
Ordinarily, we then apply the cross entropy loss 
$-\log p_{y}$.
When training privately, we use a bounded variant of the cross entropy loss,
$ -\log \psi(p_{y})$, 
where $\psi$ is defined as in \cref{affine}.

\begin{figure*}[ht]
\centering
\begin{tikzpicture}

  \begin{scope}
    \node at (0,2.8) {True Labels};
    \node[rotate=90] at (-3.5,0) {\small 0--1 error $\times$ 100};
    \node at (0,0) {\includegraphics[width=.38\linewidth]{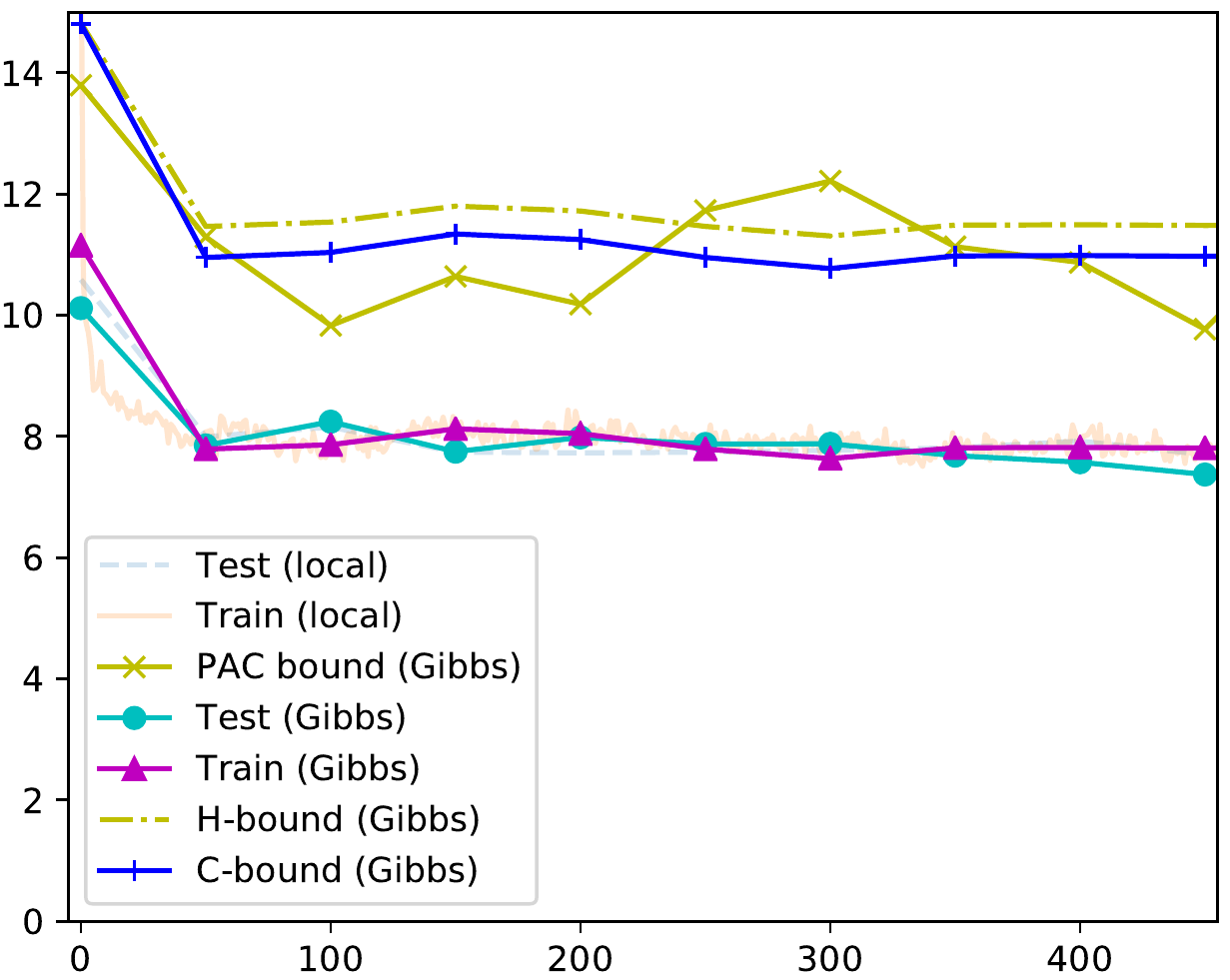}};
  \end{scope}
  \begin{scope}[xshift=.47\linewidth]
    \node at (0,2.8) {Random Labels};
    \node at (0,0) {\includegraphics[width=.38\linewidth]{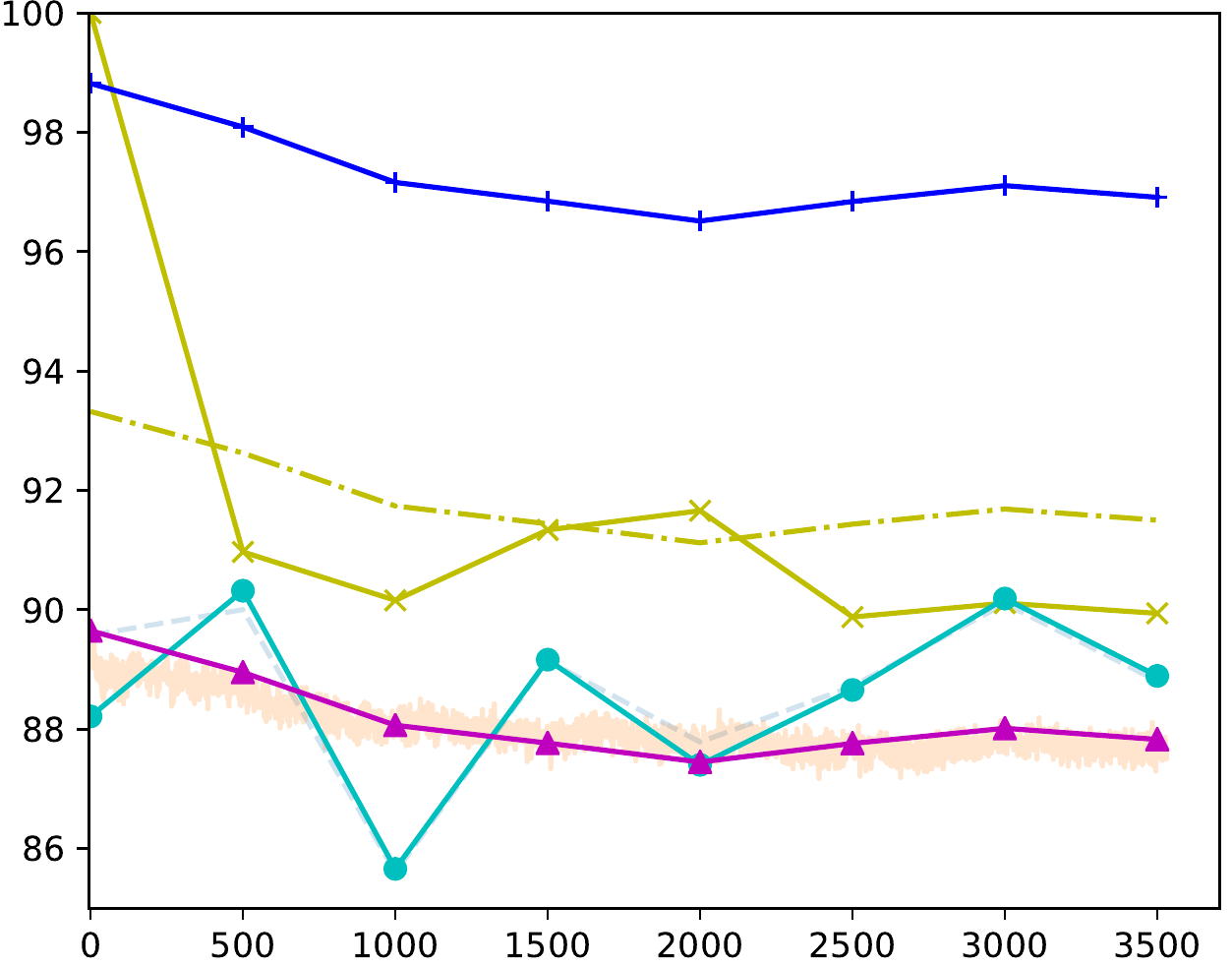}};
  \end{scope}
  \begin{scope}[yshift=-5.25cm]
  \begin{scope}
    \node at (0,0) {\includegraphics[width=.38\linewidth]{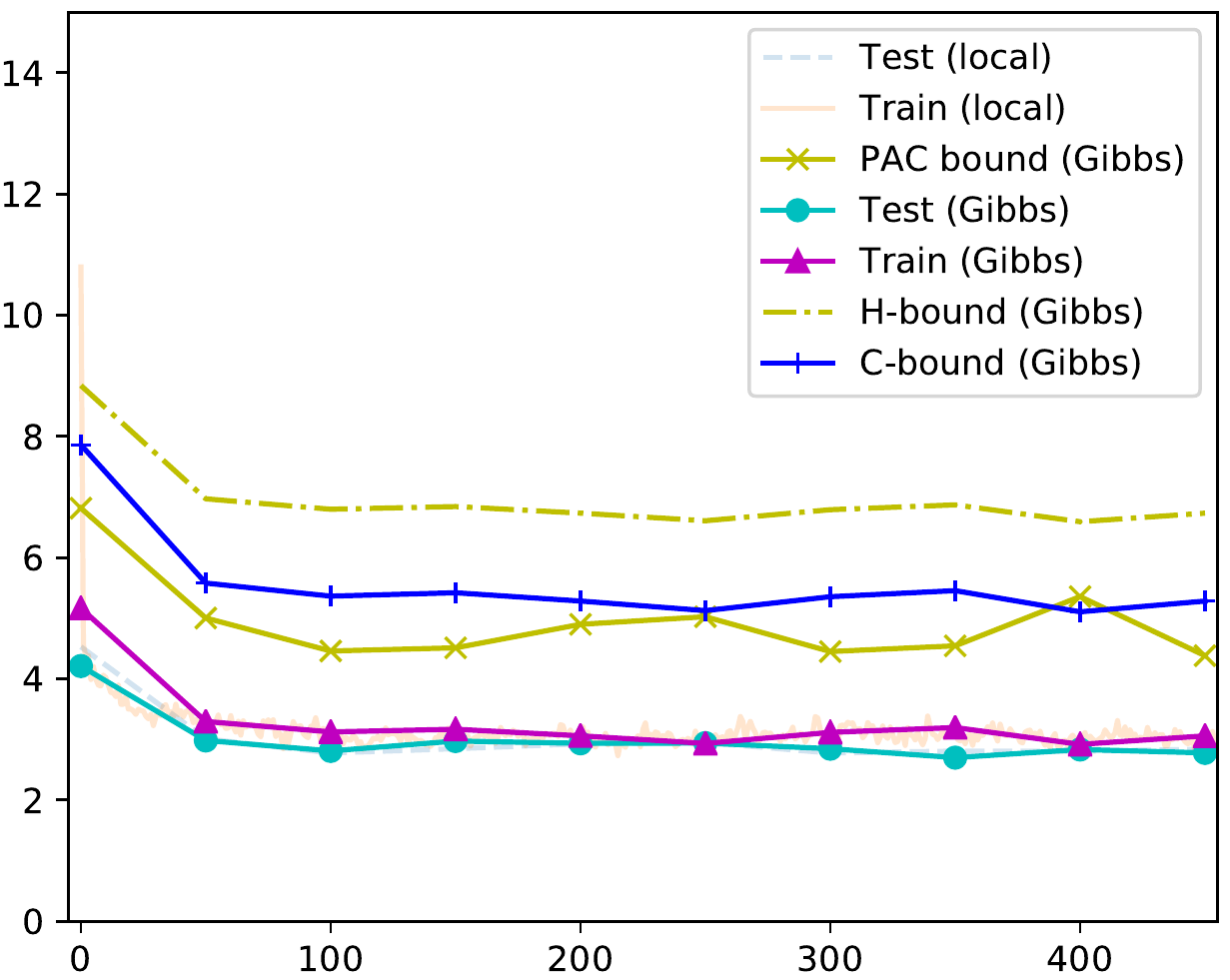}};
    \node[rotate=90] at (-3.5,0) {\small 0--1 error $\times$ 100};
    \node at (0,-2.9) {\small \Epochs};
  \end{scope}
  \begin{scope}[xshift=.47\linewidth]
    \node at (0,0) {\includegraphics[width=.38\linewidth]{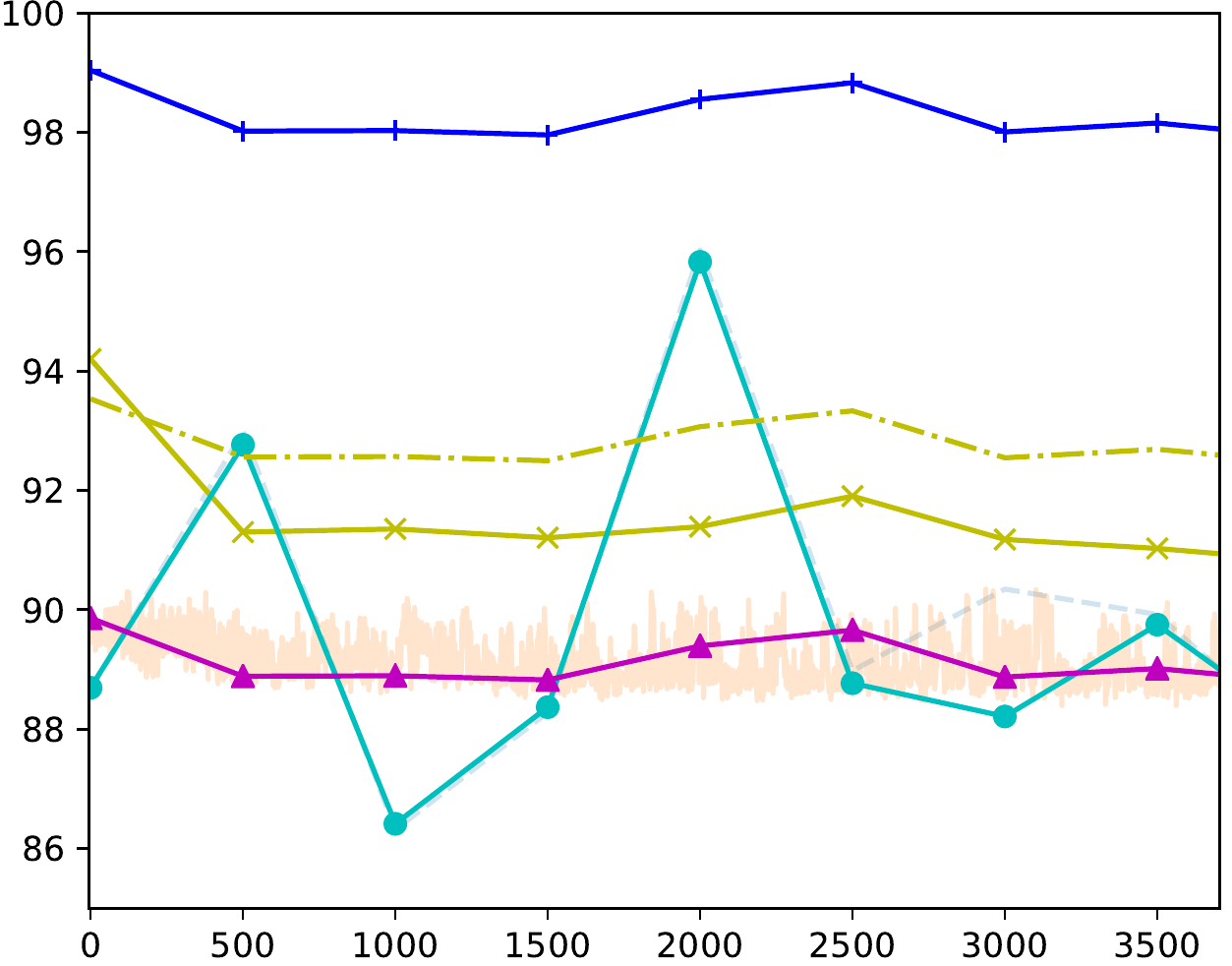}};
    \node at (0,-2.9) {\small \Epochs};
  \end{scope}
  \end{scope}
\end{tikzpicture}
\caption{
(``Local'' here refers to the mean classifier.)
Entropy-SGLD results on MNIST.
{\bf (top-left)} FC1024 network trained on true labels. 
The train and test error suggest that the generalization gap is close to zero, while all three bounds exceed the test error by slightly more than $3\%$.
{\bf (bottom-left)} CONV network trained on true labels. 
Both the train and the test errors are lower than those achieved by the FC1024 network.
We still do not observe overfitting.
The C-bound and PAC-Bayes bounds exceed the test error by $\approx 3\%$.
{\bf (top-right)} FC1024 network trained on random labels. 
After approximately 1000 epochs, we notice overfitting by $\approx 2\%$.
Running Entropy-SGLD further does not cause an additional overfitting. Theory suggests that our choice of $\tau$ prevents overfitting via differential privacy.
{\bf (bottom-right)} CONV network trained on random labels. 
We observe almost no overfitting (less than $1\%$).
Both training and test error coincide and remain close to the guessing rate ($90\%$).
}
\label{multiclass}
\end{figure*}

\section{CIFAR10 experiments} 
\label{cifar10exp}
 
We train a convolutional neural network on CIFAR10 data \citep{CIFAR10data} with true and random labels. The architecture of the network is identical to the one used in \citet{CCSL16}, but the training is performed with no dropout or weight decay.

Most of the experimental details are the same as in MNIST experiments described in \cref{multiclassmnist}. 
In particular, the training objective and the learning rate schedule is identical, with the initial outer loop base learning rate and decay both set to 0.1.  
The results reported in \cref{cifarfigtau,cifarfiggamma} are recorded after training for far more steps than necessary, 
in order to allow SGLD to get closer to its target distribution. 
In more detail, the results are obtained after 100 calls of the outerloop step (i.e., 2000 epochs), 
while we observe that the training error converges very quickly, in most cases within the first 5 calls of the outerloop \ESGLD\ step.
In order to estimate the Gibbs randomized classifier error and KL divergence and evaluate the PAC-Bayes bound, we run SGLD for an extra 50 epochs at the end of training.
 
\subsection{Privacy parameter experiments}
 
We start by fixing $\beta=1$ and experimenting with different values of the parameter $\tau$. 
Recall that the product $\tau \beta$ determines the privacy level.
The value $\tau = 10^{8}$ was used by \citet{CCSL16} in their CIFAR10 experiments.

The top row of \cref{cifarfigtau} presents the results for $\gamma = 0.03$, which is the same value used by \citet{CCSL16}.
The random labels plot highlights a phase transition for $\tau$ values in the range $10^{4}$ to $10^{5}$.
We observe that very little overfitting occurs for smaller values of $\tau$ on both true and random labels.
When $\tau$ exceeds $10^{5}$, the size of generalization gap appears to be data distribution dependent.
For random label dataset with large true Bayes error, the classifier can achieve almost zero classification error, resulting in maximal generalization error.
However, in the case of true label data, we see that the generalization error does not exceed $0.2$. 

Note, that for high values of $\tau$, the differentially private PAC-Bayes bound is completely dominated by the differential privacy penalty.
Effectively, the bound becomes data independent and thus cannot capture this difference in generalization error for true and random labels.

The bottom row contains results for $\gamma = 3$, which corresponds to shrinking the variance of the Gaussian prior and thus decreasing smoothing of the empirical error surface. 
One can recognize the same patterns as in the top row, but now the phase transition happens substantially earlier.
Due to this shift, the DP-PAC-Bayes bound approaches the C-gen bound on the generalization error.

\subsection{Prior variance experiments}

The differential privacy of sampling from the local entropy distribution does not directly depend on $\gamma$. 
However, the optimization problem is clearly affected by the value of $\gamma$, and so the performance achievable within a given privacy budget
is affected by $\gamma$. We fixed the privacy level by taking $\tau\beta= 2000$, and experimented with different $\gamma$ values.
The results are presented in \cref{cifarfiggamma}.

The left plot shows the results for $\beta=1$, which is the same value used in \cref{cifarfigtau} experiments and all MNIST experiments.
For a large range of $\gamma$ values (around 0.03 to 10), we achieve similar DP-PAC-Bayes  bound on the generalization error.
However, $\tau\in[1,3]$ yields the smallest bound on the risk and also the best performing classifiers, as judged by the risk evaluated on the test error.
A value of $\tau<0.01$ corresponds to large prior variance and excessive smoothing, which results in \ESGLD\ finding a poor classifier. 

On the right hand side plot, we reduce $\beta$ to 0.004 to be able to increase $\tau$ and maintain the same level of privacy.
This corresponds to higher SGLD noise on the outerloop step, and smaller noise on the inner SGLD step.
Remember, that the prior variance is $(\gamma \tau)^{-1}$.
Since $\tau$ is now a lot higher, $\gamma <0.001$ results in less smoothing than in the $\beta=1$ case
and we see that \ESGLD\ is now able to find a relatively good classifier while preserving the same level of privacy.
In addition, we see further improvement in the bound on the risk and the test error for a larger range of $\gamma$ values ($\gamma \in [0.0003,20]$).

 \begin{figure*}[ht]
\centering
\begin{tikzpicture}
  \begin{scope}
    \node at (0,2.8) {True Labels};
    \node[rotate=90] at (-3.5,0) {\small 0--1 error $\times$ 100};
    \node at (0,0) {\includegraphics[width=.38\linewidth]{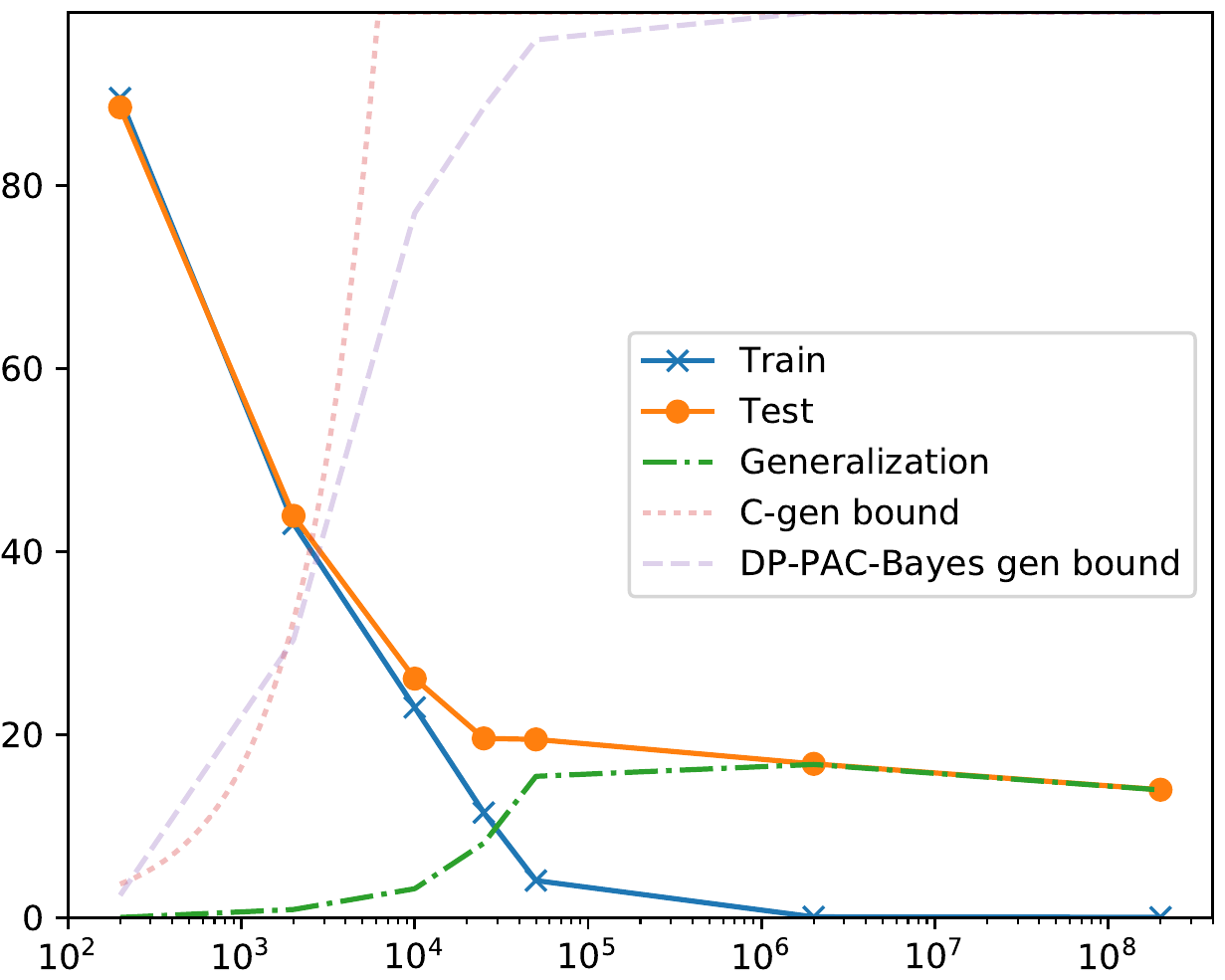}};
  \end{scope}
  \begin{scope}[xshift=.47\linewidth]
    \node at (0,2.8) {Random Labels};
    \node at (0,0) {\includegraphics[width=.38\linewidth]{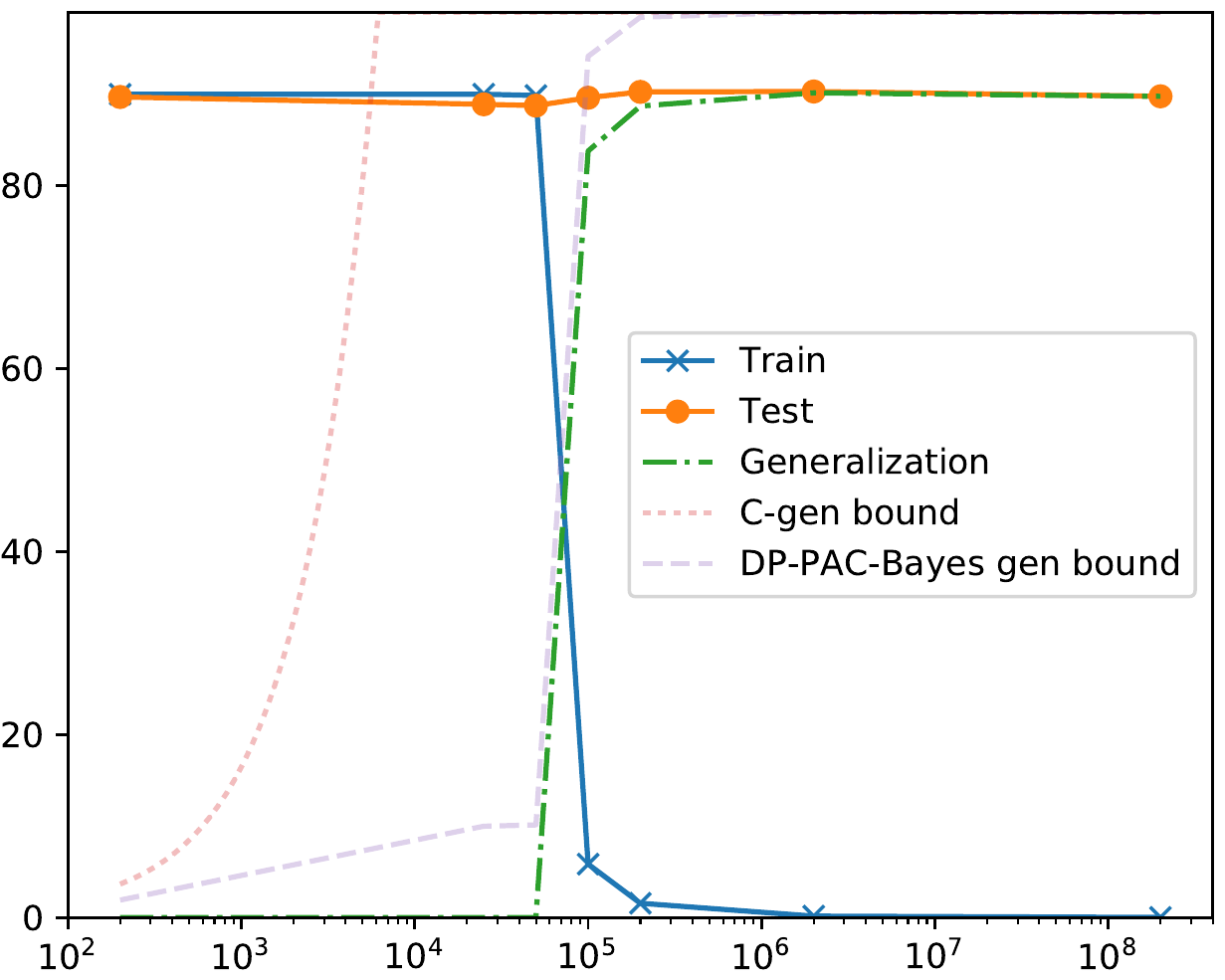}};
    \node[rotate=-90] at (3.5,0) {$\gamma=0.03$};
  \end{scope}
  \begin{scope}[yshift=-5.25cm]
  \begin{scope}
    \node at (0,0) {\includegraphics[width=.38\linewidth]{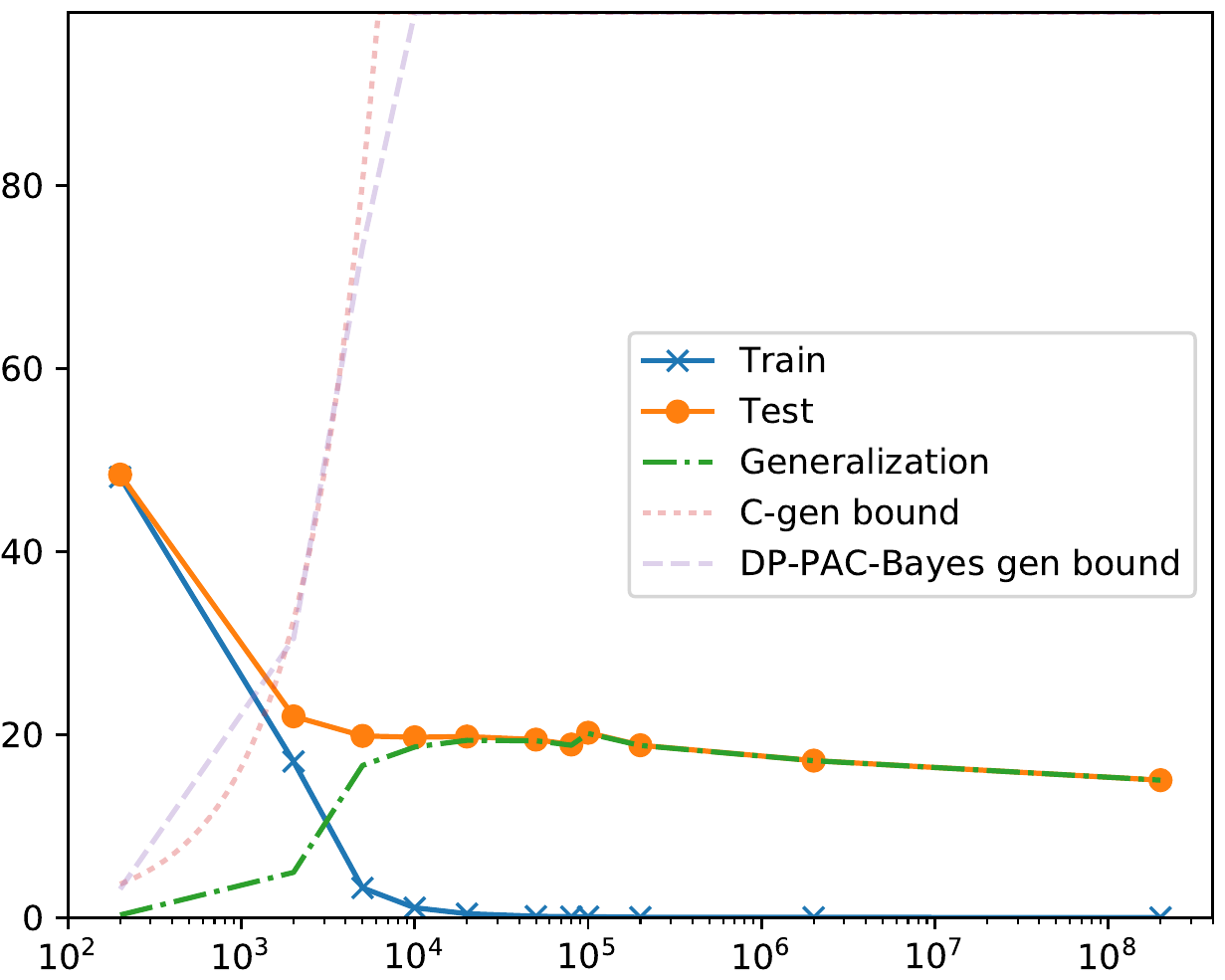}};
    \node[rotate=90] at (-3.5,0) {\small 0--1 error $\times$ 100};
    \node at (0,-2.9) {\small $\tau$};
  \end{scope}
  \begin{scope}[xshift=.47\linewidth]
    \node at (0,0) {\includegraphics[width=.38\linewidth]{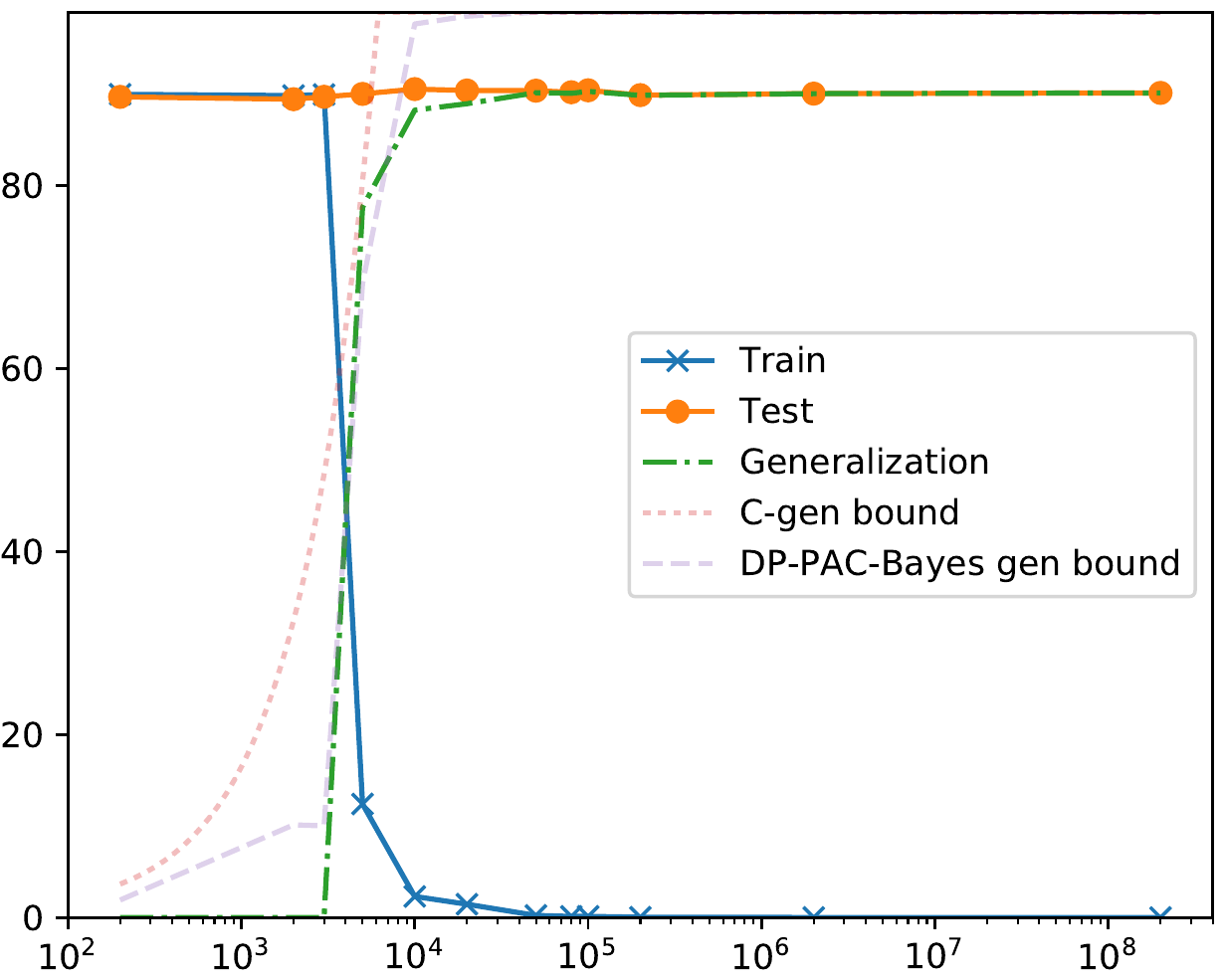}};
    \node at (0,-2.9) {\small $\tau$};
    \node[rotate=-90] at (3.5,0) {$\gamma=3$};
  \end{scope}
  \end{scope}
\end{tikzpicture}
\caption{
Results for \ESGLD\ trained on CIFAR10 data, with $\beta = 1$. 
{\bf (top-left)} 
For this configuration of parameters, \ESGLD\ finds good classifiers (with test error lower than 0.3) only at the values of $\tau$ for which both generalization bounds are vacuous. Note that, as $\tau$ increases, the test error keeps dropping, which cannot be captured by the risk bounds using differential privacy. However, this is only observed for the true label dataset, where the true Bayes error is small.
{\bf (top-right)}
The generalization gap increases with $\tau$ as suggested by the risk bounds, and takes a maximum value for $\tau>10^{6}$.
{\bf (bottom-left)} 
The pattern is similar to the $\gamma=0.03$ case (top left plot). In contrast, \ESGLD\ finds better classifiers (with test error lower than 0.3) for smaller values of $\tau$. The PAC-Bayes and C-bounds are very close to each other.
{\bf (bottom-right)} 
As in the $\gamma=0.03$ case (top right plot), we see maximal overfitting for large values of $\tau$. However, \ESGLD\ starts overfitting at a much lower $\tau$ value ($\tau > 2*10^{2}$) compared to the smaller $\gamma$ case ($\tau > 5*10^{4}$). The PAC-Bayes bound approaches C-bound but no generalization bounds are violated.
}\label{cifarfigtau}
\end{figure*}

\begin{figure*}[ht]
\centering
\begin{tikzpicture}[]
  \begin{scope}[xshift=.1\linewidth]
  \begin{scope}
    \node at (0,2.8) {$\beta = 1, \,\tau = 2000$};
    \node[rotate=90] at (-3.5,0) {\small 0--1 error $\times$ 100};
    \node at (0,-2.9) {\small $\gamma$};
    \node at (0,0) {\includegraphics[width=.38\linewidth]{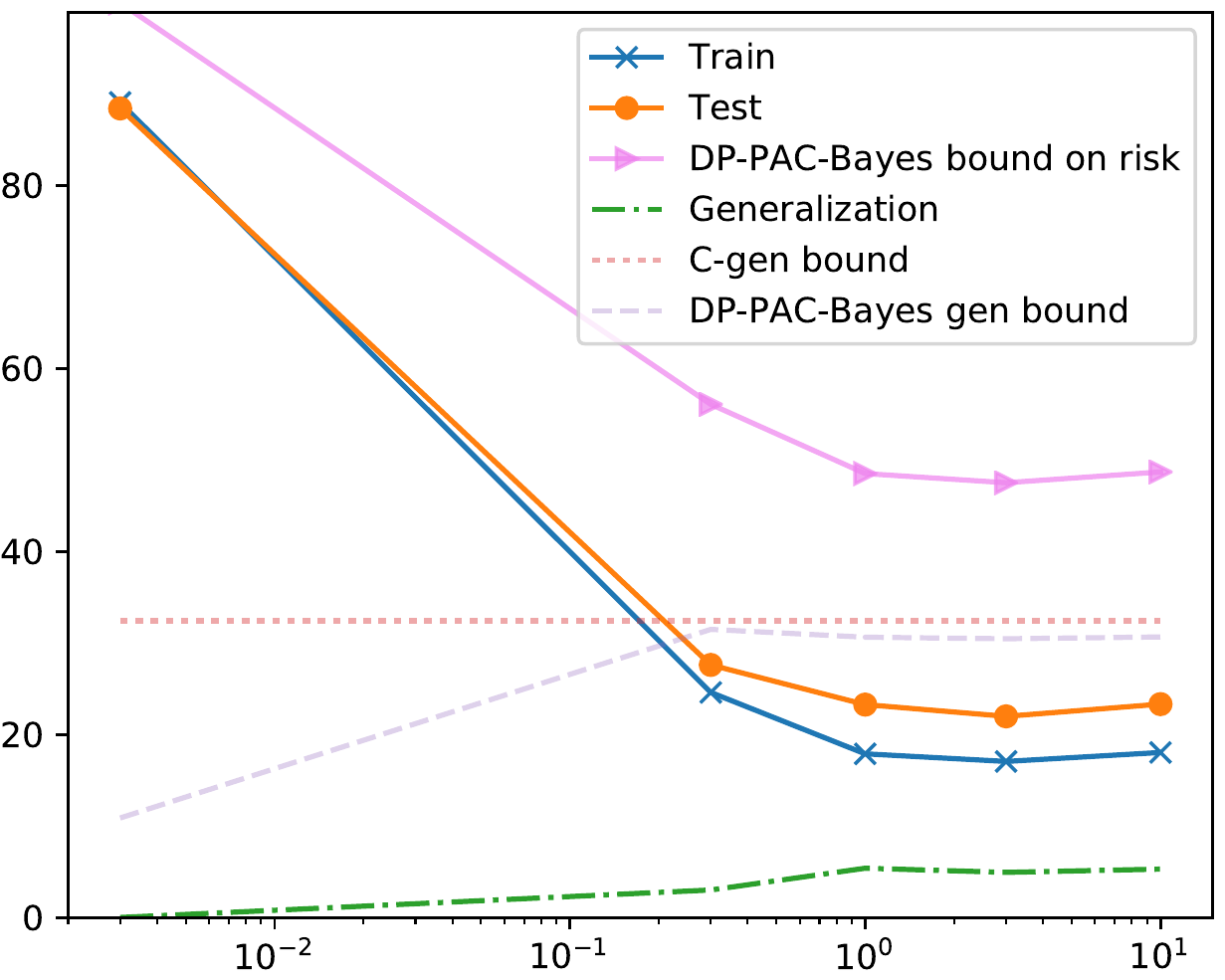}};
  \end{scope}
  \begin{scope}[xshift=.47\linewidth]
    \node at (0,0) {\includegraphics[width=.38\linewidth]{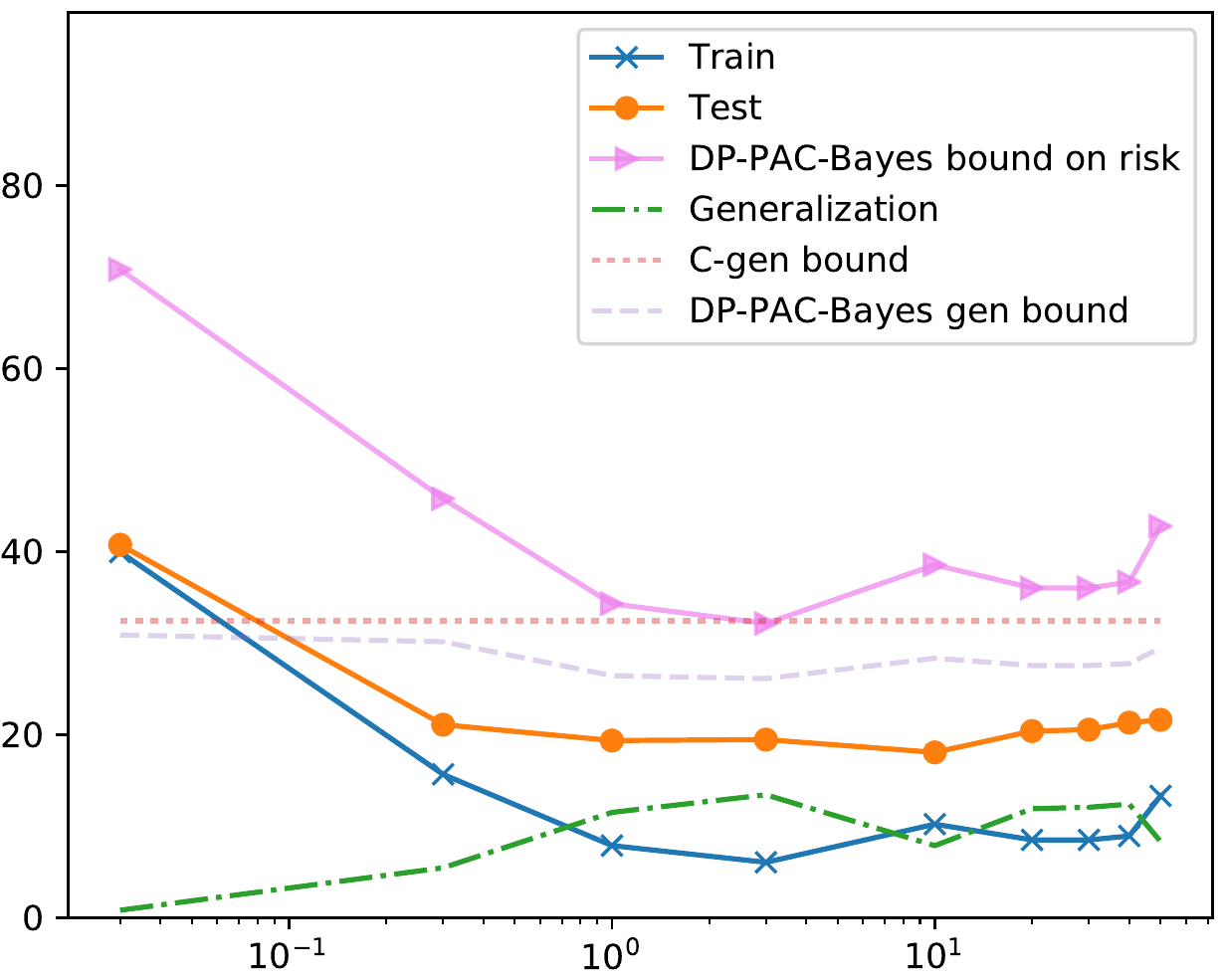}};
    \node at (0,2.8) {$\beta = 4e-3,\, \tau = 5e5$};
    \node[rotate=-90] at (3.5,0) {True Labels};
    \node at (0,-2.9) {\small  $\gamma$};
  \end{scope}
  \end{scope}
\end{tikzpicture}
\caption{
\ESGLD\ experiments on CIFAR10 data with a fixed level of privacy $\beta \tau = 2000$. The generalization bounds we evaluate do not depend on $\gamma$, therefore the C-gen bound takes a constant value. The DP-PAC-Bayes generalization bound corresponds to DP-PAC-Bayes bound on risk minus the empirical error (Train), and is tighter than the C-gen bound for all values of $\beta,\, \tau, \,\gamma$ that we tested.
{\bf (left)} $\beta$ is set to $1$ during optimization. A small value of $\gamma=0.003$ corresponds to high prior variance and thus  substantial smoothing of the optimization surface. In this case, \ESGLD\ fails to find a good classifier. We hypothesize that this happens due to over-smoothing and information loss relating to the location of the actual empirical risk minima.
{\bf (right)} $\beta$ is reduced to $0.004$, which allows us to increase $\tau$ and maintain the same level of privacy. One can notice that we are able to achieve lower empirical risk (train error is $\approx 0.06$ for $\gamma=3$), risk ($\approx 0.19$ on the test set), and thus a lower PAC-Bayes bound on the risk ($\approx 0.32$) compared to the best result for $\beta=1$ setup.
The largest value of $\gamma$ tested is $50$. In this case, the prior variance $(\tau \gamma)^{-1} $ is very small, and the step size is also very small (initial step size is set to base learning rate times the prior variance). This may explain why \ESGLD\ finds slightly worse classifiers compared to smaller values of $\gamma$.
}\label{cifarfiggamma}
\end{figure*}

\section{Related work}
\label{sec:rel}

Entropy-SGD connects to early work by \citet{Hinton93} and \citet{Hochreiter97}.
Both sets of authors introduce regularization schemes based on minimum-description-length principles.
\citeauthor{Hinton93} aim to find weights with less information by finding weights whose empirical risk is insensitive to the addition of Gaussian noise;
 \citeauthor{Hochreiter97} define an algorithm that explicitly seeks out ``flat minima'', 
i.e., a large connected region where the empirical risk surface is nearly constant.
Building on work by \citet{LanCar2002}, these ideas were recently revisited by \citet{DR17}. 
In their work,
a PAC-Bayes bound is minimized with respect to the posterior using nonconvex optimization.
(\citet{AchilleS17} arrive at a similar objective from the information bottleneck perspective.)
In contrast, our work shows that Entropy-SGD implicitly uses the optimal (Gibbs) posterior, 
and then optimizes the resulting PAC-Bayes bound with respect to the prior.

Our work also relates to renewed interest in nonvacuous generalization bounds \citep{LangfordPHD,LanCar2002},
i.e., bounds on the numerical difference between the unknown classification error and the training error that are (much) tighter than the tautological upper bound of one. 
In the same work by \citet{DR17} discussed above, 
nonvacuous generalization bounds are demonstrated for MNIST.
(Their algorithm can be viewed as variational dropout \citep{kingma15} or information dropout \citep{achille2018information}, 
with a proper data-independent prior but without local reparametrization.)
Their work builds on core insights by \citet{LanCar2002}, 
who computed nonvacuous bounds for neural networks five orders of magnitude smaller.
Our work shows that Entropy-SGLD yields generalization bounds, though the value of these bounds depends on
the degree to which SGLD is allowed to converge. Under the optimistic assumption that convergence has been achieved, 
we see that the resulting bounds are tighter than those computed by \citeauthor{DR17}.

Our analysis of \ESGLD\ exploits results in differential privacy \citep{Dwork2006}
and its connection to generalization \citep{dwork2015preserving,dwork2015generalization,bassily2016algorithmic,Oneto2017}.
\ESGLD\ is related to differentially private empirical risk minimization, 
which is well studied, both in the abstract
\citep{chaudhuri2011differentially, kifer2012private,bassily2014differentially}
and in the particular setting of private training via SGD \citep{bassily2014differentially, Abadi:2016}.
Given the connection between Gibbs posteriors and Bayesian posteriors,
\ESGLD\ also relates to the differential privacy of Bayesian and approximate sampling algorithms
\citep{mir2013differential, bassily2014differentially, dimitrakakis2014robust, Wang:2015, Minami16}.

\citet{London2017} introduces PAC-bayes bounds where the prior and posterior are not defined on a hypothesis space, but rather on the randomness source used by a (randomized) learning algorithm, such as SGD. Like the bounds we use, London's bounds depend on the stability of a learning algorithm as well as the relative entropy from posterior to prior. The priors in London's setting, however, are chosen before seeing the data, as usual, but unlike in our analysis of Entropy-SGLD, where we investigate private data-dependent priors. We believe these uses of privacy/stability are complementary.

Finally, the local entropy should not be confused with the smoothed risk surface
\[\textstyle
\ww \mapsto \int f(\ww+x) \Normal\!(\dee x),
\]
i.e., obtained by convolution with a Gaussian kernel: in that case, every point on this surface represents the risk of a randomized classifier, obtained by perturbing the parameters according to a Gaussian distribution. (This type of smoothing relates to the approach of \citet{DR17}.)
The local entropy 
also relates to a Gaussian perturbation, 
but the perturbation is either accepted or rejected based upon its relative performance (as measured by the exponentiated loss) compared with typical perturbations.
Thus the local entropy perturbation concentrates on regions of weight space with low empirical risk,
provided they have sufficient probability mass under the distribution of the random perturbation.

\end{document}